\documentclass[11pt]{article}

\usepackage{geometry}
\usepackage{subcaption}
\usepackage{cite}
\usepackage{xspace}
\usepackage{booktabs} 
\usepackage[toc,page]{appendix}
\usepackage{amssymb}
\usepackage{amsmath}
\usepackage{amsthm}
\usepackage{mathtools} 
\usepackage{comment}
\usepackage{multirow}
\usepackage{multicol}
\usepackage[export]{adjustbox}
 \usepackage{array}
 \usepackage{arydshln}
 \usepackage{makecell}
 \usepackage{algorithm}
\usepackage{algorithmic}
\usepackage{enumitem}
\usepackage{xcolor}

\usepackage{apptools}
\AtAppendix{\counterwithin{theorem}{section}} 

\usepackage{comment}
\usepackage{caption}
\usepackage{subcaption}
\usepackage{xcolor}
\usepackage{multirow}
\usepackage{dsfont}
\usepackage{breakcites}

\usepackage{amsmath,amssymb,amsthm,amsfonts}
\usepackage{caption}
\usepackage{subcaption}
\usepackage{latexsym,bbm,graphicx,float,mathtools}

\usepackage[colorlinks,linkcolor=blue,citecolor=blue]{hyperref}
\usepackage[capitalize,nameinlink]{cleveref}

\usepackage{amsmath,amsfonts,bm}

\def\eqref#1{equation~\ref{#1}}

\def\1{\bm{1}}

\DeclareMathAlphabet{\mathsfit}{\encodingdefault}{\sfdefault}{m}{sl}
\SetMathAlphabet{\mathsfit}{bold}{\encodingdefault}{\sfdefault}{bx}{n}

\newcommand{\E}{\mathbb{E}}

\newcommand{\R}{\mathbb{R}}

\DeclareMathOperator{\sign}{sign}

\newcommand{\Ex}{\mathop{{\bf E}\/}}
\renewcommand{\Pr}{\operatorname{{\bf Pr}}}

\newcommand{\Ber}{\operatorname{Ber}}
\newcommand{\poly}{\mathrm{poly}}

\newtheorem{theorem}{Theorem}
 \newtheorem{lemma}{Lemma}[section]
\newtheorem{claim}[lemma]{Claim}
\newtheorem{proposition}[theorem]{Proposition}

\newtheorem*{theorem*}{Theorem}
\newtheorem*{proposition*}{Proposition}

\begin{document}

\title{On Classification Thresholds for Graph Attention with Edge Features}

\author{
        \name Kimon~Fountoulakis
        \AND
        \name Dake~He
        \AND
	    \name Silvio~Lattanzi
	    \AND
	    \name Bryan~Perozzi
	    \AND
	    \name Anton~Tsitsulin
	    \AND
	    \name Shenghao~Yang
}

\author{
	Kimon~Fountoulakis%
        \thanks{School of Computer Science, University of Waterloo, and Google Waterloo, Waterloo, ON, Canada. E-mail: kfountou@uwaterloo.ca.
        }
        \and
        Dake~He%
        \thanks{Google Waterloo, Waterloo, ON, Canada. E-mail: dkhe@google.com.
        } 
        \and
	    Silvio~Lattanzi%
        \thanks{Google Research, Zurich, Switzerland. E-mail: silviol@google.com.
        }
        \and
	    Bryan~Perozzi%
        \thanks{Google Research, New York, USA. E-mail: bperozzi@acm.org.
        }
        \and
	    Anton~Tsitsulin%
        \thanks{Google Research, New York, USA. E-mail: tsitsulin@google.com.
        }
        \and
	    Shenghao~Yang%
        \thanks{School of Computer Science, University of Waterloo, Waterloo, ON, Canada. E-mail: s286yang@uwaterloo.ca.
        }
}

\maketitle

\begin{abstract}

The recent years we have seen the rise of graph neural networks for prediction tasks on graphs. One of the dominant architectures is graph attention due to its ability to make predictions using weighted edge features and not only node features. In this paper we analyze, theoretically and empirically, graph attention networks and their ability of correctly labelling nodes in a classic classification task. More specifically, we study the performance of graph attention on the classic contextual stochastic block model (CSBM). In CSBM the nodes and edge
features are obtained from a mixture of Gaussians and the edges from a stochastic block model. We consider a general graph attention mechanism that takes random edge features as input to determine the attention coefficients. We study two cases, in the first one, when the edge features are noisy, we prove that the majority of the attention coefficients are up to a constant uniform. This allows us to prove that graph attention with edge features is not better than simple graph convolution for achieving perfect node classification. Second, we prove that when the edge features are clean graph attention can distinguish intra- from inter-edges and this makes graph attention better than classic graph convolution. 


\end{abstract}
\section{Introduction}\label{sec:intro}

Learning from multi-modal datasets is currently one of the most prominent topics in artificial intelligence. The reason behind this trend is that many applications, such as recommendation systems, fraud detection and vision, require some combination of different types of data. 
In this paper we are interested in multi-modal data which combine a graph, i.e., a set of nodes and edges, with attributes for each node and edge. The attributes of the nodes/edges capture information about the nodes/edges themselves, while the edges among the nodes capture relations among the nodes. Capturing relations is particularly helpful for applications where we are trying to make predictions for nodes given neighborhood data. 

One of the most prominent ways of handling multi-modal data for downstream tasks such as node classification are graph neural networks~\cite{GMS05,scarselli:gnn,BZSL14,DMAGHAA15,HBL15,AT16,DBV16,HYL17,kipf:gcn}. Graph neural network models can mix hand-crafted or automatically learned attributes about the nodes while taking into account relational information among the nodes. Their output vector representation contains both global and local information for the nodes. This contrasts with neural networks that only learn from the attributes of entities. 

\subsection{Motivation and goals}
Graph neural networks have found a plethora of uses in chemistry~\cite{gilmer:quantum, scarselli:gnn}, biology, and in various industrial applications. Some representative examples include fighting spam and abusive behaviors, providing personalization for the users~\cite{YHCEHL18}, and predicting states of physical objects~\cite{battaglia:graphnets}. Given wide applicability and exploding popularity of GNNs, theoretically understanding in which regimes they work best is of paramount importance.

One of the most popular graph neural network architectures is the Graph Attention Network (GAT). Graph attention~\cite{Velickovic2018GraphAN} is usually defined as
averaging the features of a node with the features of its neighbors by appropriately weighting the edges of a graph before spatially convolving the node features. It is generally expected by practitioners that GAT is able to downweight unimportant edges and set a large weight for important edges, depending on the downstream task. In this paper we analyze the graph attention mechanism. 

We focus on node classification, which is one of the most popular tasks for graph learning. We perform our analysis using the contextual stochastic block model (CSBM)~\cite{BVR17,DSM18}. The CSBM is a coupling of the stochastic block model (SBM) with a Gaussian mixture model. We focus on two classes where the answer to the above question is sufficiently precise to understand the performance of graph attention and build useful intuition about it.

We study perfect classification as it is one of the three questions that has been asked for the community detection for SBM without node features~\cite{Abbe2018}. We leave results on other types of classification guarantees for future work. Our goal is study the performance of graph attention on a well-studied synthetic data model. We see our paper as a small step in the direction of building theoretically justified intuition about graph attention and better attention mechanisms.

\subsection{Contributions} 
We study the performance of graph attention with edge and node features for the CSBM. The edge features follow a Gaussian mixture model with two means, one for intra-edges and one for inter-edges. We call the edge features clean when the distance between the means is larger than the standard deviation. We call the edge features noisy when the distance between the means is smaller than the standard deviation. We split our results into two parts. In the first part we consider the case where the edge features that are passed to the attention mechanism are clean. In the second part we consider the case where the edge features are noisy. We describe our contributions below.

\begin{enumerate}
    \item \textit{Separation of intra and inter attention coefficients for clean edge features}. There exists an attention architecture which can distinguish intra- from inter-edges. This attention architecture allows us to prove that the means of the convolved data do not move closer, while achieving large variance reduction. It also allows us to prove that the threshold of perfect node classification for graph attention is better than that of graph convolution.
    \item \textit{Perfect node classification for clean edge features}. Let $\sigma$ be the standard deviation of the node features, $n$ the number of nodes and $p,q$ the intra- and inter-edge probabilities. If the distance between the means of the node features is $\omega\left(\sigma \sqrt{\frac{\log{n}}{n \max(p,q)}}\right)$. Then with high probability graph attention classifies the data perfectly.
    \item \textit{Failure of perfect node classification for clean edge features}. If the distance between the means of the node features is small, that is smaller than $K \sigma  \sqrt{\frac{\log{n}}{n\max(p,q)}\left(1-\max(p,q)\right)}$ for some constant $K$, then graph attention can't classify the nodes perfectly with probability at least $1-1/n^{1-\max(p,q)}$.
    \item \textit{Uniform intra and inter attention coefficients for noisy edge features.} We prove that for $n-o(n)$ nodes at least $90\%$ of their attention coefficients are up to a constant uniform. This means that a lot of attention coefficients are up to a constant the same as those of graph convolution. This property allows us to show that in this regime graph attention is not better than graph convolution.
    \item \textit{Perfect node classification for noisy edge features.} If the distance is $ \omega\left(\sigma\frac{p+q}{|p-q|}\sqrt{\frac{\log n}{n\max(p,q)}}\right)$, then with high probability graph attention classifies all nodes correctly. 
    \item \textit{Failure of perfect node classification for noisy edge features.} If the distance is less than $K\sigma\frac{p+q}{|p-q|}\sqrt{\frac{\log n}{n\max(p,q)}(1-\max(p,q))}$ for some constant $K$, then graph attention can't classify the nodes perfectly with probability at least $1-1/n^{1-\max(p,q)}$.
\end{enumerate}

Finally we complement our theoretical results with an empirical analysis confirming our main findings.

\section{Relevant Work}\label{sec:prev_work}

There have been numerous papers proposing new graph neural network architectures that is impossible to acknowledge all works in one paper. We leave this work for relevant books and survey papers on graph neural networks, examples include~\cite{HamilBook,WPCLZY21}. From a theoretical perspective, a few authors have analyzed graph neural networks using traditional machine learning frameworks or from a signal processing perspective~\cite{CLB19,Chien:2020:joint,Zhu:2020:generalizing,XHLJ19,GJJ20,A2020,ALoukas2020}. For a recent survey in this direction see the recent survey paper~\cite{J22} that focuses on three main categories, representation, generalization and extrapolation. In our paper we analyze graph attention from a statistical perspective that allows us to formally understand claims about graph attention. 

In the past, researchers have put significant effort in understanding community detection for the SBM~\cite{Abbe2018}. Usually the results for community detection are divided in three parameters regimes for the SBM. The first type of guarantee that was investigated was that of exact recovery or perfect classification. 
We are also interested in perfect node classification, but our work is on graph attention for the CSBM.
The analysis of exact recovery in SBM and perfect classification in CSBM for graph attention are significantly different. In fact, our focus is not on designing the best algorithm for the exact classification task but it is on understanding the advantages and limitation of Graph Attention over other standard architectures. As a consequence, the model we analyze is a non-linear function of the input data since we have to deal with the coupling of highly nonlinear attention coefficients, the node features and the graph structure. 

A closely related work is~\cite{BFJ2021}, which studies the performance of graph convolution~\cite{kipf:gcn} on CSBM as a semi-supervised learning problem. In our paper we work with graph attention and we compare it to graph convolution. Another relevant work is~\cite{fountoulakis2022graph}. In this paper the authors also study the performance of graph attention for CSBM. However, in~\cite{fountoulakis2022graph} edge features are not used and there is no result provided about when graph attention fails to achieve perfect node classification, only a conjecture is provided. In this paper we provide a complete treatment regarding the question of perfect classification when edge features are given. Another paper that studies performance of graph attention on CSBM is~\cite{anonymous2023results}. In this paper the attention architecture is constructed using ground-truth labels and it is fixed. The authors also consider an attention architecture which is constructed using an eigenvector of the adjacency matrix when the community structure in the graph can be exactly recovered. Thus in~\cite{anonymous2023results} only a rather optimistic scenario is studied, that is, when we are given a good attention architecture. In our paper, we consider the case where additional edge features are given that follow a Gaussian mixture model and we analyze the performance of graph attention when these features are clean or noisy. We provide complete analysis about the attention coefficients, instead of assuming them, and we show how they affect perfect node classification when the edge features are clean or noisy.

Within the context of random graphs another relevant work is~\cite{KBV21,MLLK22}. In the former paper, the authors study universality of graph neural networks on random graphs. In the latter paper the authors go a step further and prove that the generalization error of graph neural networks between the training set and the true distribution is small, and the error decreases with respect to the number of training samples and the average number of nodes in the graphs. In our paper we are interested in understanding the parameters regimes of CSBM such that graph attention classifies or fails to classify the data perfectly. This allows us to compare the performance of graph attention to other basic approaches such as a graph convolution. 

Other papers that have studied the performance of graph attention are~\cite{BAY21,KTA19,hou2019measuring}. In~\cite{BAY21} the authors show that graph attention fails due to a global ranking of nodes that is generated by the attention mechanism in~\cite{Velickovic2018GraphAN}. They propose a deeper attention mechanism as a solution. Our analysis a deeper attention mechanism is not required since we consider independently distributed edge features and the issue mentioned in~\cite{BAY21} is avoided.

The work in~\cite{KTA19} is an empirical study of the ability of graph attention to generalize on larger, complex, and noisy graphs. Finally, in~\cite{hou2019measuring} the authors propose a different metric to generate the attention coefficients and show empirically that it has an advantage over the original GAT architecture. In our paper we consider the original and most popular attention mechanism~\cite{Velickovic2018GraphAN} and its deeper variation as well.
\section{Preliminaries}\label{sec:notation}

In this section we describe the data model that we use and the graph attention architecture.

\subsection{The contextual stochastic block model with random edge features}

In this section we describe the CSBM~\cite{DSM18}, which is a simple coupling of a stochastic block model with a Gaussian mixture model. Let $(\epsilon_{k})_{k\in[n]}$ be i.i.d \ Bernoulli random variables. These variables define the class membership of the nodes. In particular, consider a stochastic block model consisting of two classes
$C_{0}=\{i\in[n]:\epsilon_{i}=0\}$ and $C_{1}=C_{0}^c$ with inter-class edge probability $q$ and intra-class edge probability $p$ with no self-loops\footnote{In practice, self-loops are often added to the graph and the following adjacency matrix $\tilde{A}$ is used instead: $\tilde{A}=(\tilde{a}_{ij})$ is the matrix $A + I$ and $D$ to be the diagonal degree matrix for $\tilde{A}$, so that $D_{ii} = \sum_{j\in[n]}\tilde{a}_{ij}$ for all $i\in[n]$. Our results can be extended to this case with minor changes.}.
In particular, given $(\epsilon_{k})$ the adjacency matrix $A=(a_{ij})$ follows a Bernoulli distribution where $a_{ij}\sim \Ber(p)$ if $i,j$ are in the same class  and $a_{ij}\sim \Ber(q)$ if they are in distinct classes. This completes the distributions for the class membership and the graph. Let us now describe the distributions for the node and edges features.

Consider the node features $x_{i}$ to be independent $d$-dimensional Gaussian random vectors with
$x_{i}\sim N(\mu,\sigma I)$ if $i\in C_{0}$ and
$x_{i}\sim N(-\mu,\sigma I)$ if $i\in C_{1}$.
Here $\mu\in\R^{d}$ is the mean, $\sigma\ge 0$ is the standard deviation and $I$ is the identity matrix. 
Let $\mathcal{E}$ be the set of edges which consists of pairs $(i,j)$ such that $a_{ij}=1$. Consider $E\in\R^{|\mathcal{E}|\times h}$ to be the edge feature matrix such that such that each row $E_{(i,j)}$ is an independent $h$-dimensional Gaussian random vector with
$E_{(i,j)}\sim N(\nu,\zeta I)$ if $(i,j)$ is an intra-edge, i.e., $i,j\in C_{0}$ or $i,j\in C_{1}$, and
$E_{(i,j)}\sim N(-\nu,\zeta I)$ if $(i,j)$ is an inter-edge, i.e., $i\in C_{0}, j\in C_1$ or $i\in C_{1}, j\in C_0$.
Here $\nu\in\R^{d}$ is the mean, $\zeta\ge 0$ is the standard deviation.

Denote by $CSBM(n,p,q,\mu,\nu,\sigma, \zeta)$ the coupling of a stochastic block model with the Gaussian mixture models for the nodes and the edges with means $\mu,\nu$ and standard deviation $\sigma,\zeta$, respectively, as described above. We denote a sample by $(A,X, E)\sim CSBM(n,p,q,\mu,\nu,\sigma, \zeta)$.

\subsection{Assumptions}
We now we state two standard assumptions on the CSBM that we will use in our analysis. The first assumption is $p,q\ge \Omega (\log^2 n / n )$, and it guarantees that the expected degrees of the graph are $\Omega(\log^2 n)$, they also guarantee degree concentration. The second assumption is that the standard deviation $\zeta$ of the edge features is constant. This assumption is without loss of generality since all that really matters is the ratio of the distance between the means of the edges features over the standard deviation. As long as we allow the distance between the means to grow while $\zeta$ is fixed then the results are not restricted, while the analysis is simplified.

\subsection{Graph attention}

The graph attention convolution is defined as $\hat x_i=\sum_{j\in [n]} \tilde{A}_{ij}\gamma_{ij} x_j$ $\forall i\in[n]$, 
where $\gamma_{ij}$ is the attention coefficient of the edge $(i,j)$. We focus on a single layer graph attention since this architecture is enough for the simple CSBM that we consider. 

There are many ways to set the attention coefficients $\gamma$. We discuss the setting in our paper and how it is related to the original definition in~\cite{Velickovic2018GraphAN} and newer ones~\cite{BAY21}. We define the attention function $\Psi(E_{(i,j)})$ which takes as input the features of the edge $(i,j)$ $E_{(i,j)}$ and outputs a scalar value. The function $\Psi$ is often parameterized by learnable variables, and it is used to define the attention coefficients
\begin{align*}
    \gamma_{ij}:= \frac{\exp(\Psi(E_{(i,j)}))}{\sum_{\ell\in N_i}\exp(\Psi(E_{(i,\ell)}))}\label{def:attention_coeff},
\end{align*}
where $N_i$ is the set of neighbors of node $i$. 

In the original paper~\cite{Velickovic2018GraphAN} the function $\Psi$ is a linear function of the two dimensional vector $[w^T x_i,w^T x_j]$ passed through LeakyRelu, where the coefficients of the linear function are learnable parameters, $w$ are learnable parameters as well and are shared with the parameters outside attention. In this paper we consider independent edge features as input to the attention mechanism. Although in the original paper~\cite{Velickovic2018GraphAN} edge features are mentioned as an input to the model this seems an important departure from what was extensively studied in~\cite{Velickovic2018GraphAN}. However, using edge features captures the effect of dominating noise in graph attention, which is what we are interested in this paper for understanding performance of graph attention. 
Finally, we consider $\Psi$ functions that are a composition of a Lipschitz and a linear function. This is enough to prove that graph attention is able to distinguish intra- from inter-edges and consequently leads to better performance than graph convolution when the edge features are clean. Given that the edge features in our data model are independent from node features, this setting avoids the issues discussed in~\cite{BAY21}.



\section{Results}\label{sec:results}

In this section we describe our results. We split the section into two parts. In the first part we describe performance of graph attention in case the edge features are clean. In the second part we describe performance of graph attention in case the edge features are noisy.

\subsection{Clean edge features}

Consider the case of clean edge features, this means that $\|\nu\|\ge \omega(\zeta \sqrt{\log (|\mathcal{E}|)})$. We call this regime clean because in this case there is not much overlap between the two Gaussians in the Gaussian mixture model of the CSBM. In the following theorem we prove that there exists an attention architecture such that it is able to distinguish intra- from inter-edges. The reason that such an attention architecture is useful is because it allows us to prove in~\cref{thm:clean} that the means of the convolved node features do not move closer, while achieving large variance reduction. The importance of such an attention mechanism is also verified by the fact that using it the threshold of perfect node classification in~\cref{thm:clean} is better than that of graph convolution. We comment on this later on in this section.

\begin{proposition}\label{prop:gammas}
Let $(A,X, E)\sim CSBM$, and assume that $\|\nu\|\ge \omega(\zeta \sqrt{\log (|\mathcal{E}|)})$. If $p>q$, we have that there exists a function $\Psi$ that provides the following attention coefficients
    \begin{equation*}
    \gamma_{ij}=
        \begin{cases}
            \frac{2}{n p}(1\pm o_n(1)) & i,j\in C_0 \mbox{ or } i,j\in C_1\\
            o\left(\frac{1}{n(p+q)}\right) & \mbox{otherwise.}\\
        \end{cases}
    \end{equation*}
    with probability $1-o_{n}(1)$. If $q>p$, we have that
    \begin{equation*}
    \gamma_{ij}=
        \begin{cases}
            o\left(\frac{1}{n(p+q)}\right) & i,j\in C_0 \mbox{ or } i,j\in C_1\\
            \frac{2}{n q}(1\pm o_n(1)) & \mbox{otherwise.}\\
        \end{cases}
    \end{equation*}
    with probability $1-o_{n}(1)$.
\end{proposition}

\textit{Proof sketch}. We construct $\Psi$ such that it separates intra- from inter-edges and it concentrates around its mean. We define $s:=\sign(p-q)\nu/\|\nu\|$ and  $\Psi(E_{(i,j)}):= s^T E_{(i,j)}$, which measures correlations with one of the means of the Gaussian mixture for the edge features. If $p>q$ the $\Psi$ function concentrates around a large positive value for intra-edges and a large negative value for inter-edges. The opposite holds for $q>p$. Then we plug in $\Psi$ in the definition of the attention coefficients $\gamma_{ij}$. Using concentration of $\Psi$ we prove the result.

In the following theorem we utilize~\cref{prop:gammas} to prove a positive result and a negative result for perfect classification using graph attention. 

\begin{theorem}\label{thm:clean}
Let $(A,X, E)\sim CSBM$, and assume that $\|\nu\|\ge \omega(\zeta \sqrt{\log (|\mathcal{E}|)})$.
\begin{enumerate}
    \item If $\|\mu\|\ge \omega\left(\sigma \sqrt{\frac{\log{n}}{n \max(p,q)}}\right)$, then we can construct a graph attention architecture that classifies the nodes perfectly with probability $1-o_n(1)$.
    \item If $\|\mu\|\le  K \sigma  \sqrt{\frac{\log{n}}{n\max(p,q)}\left(1-\max(p,q)\right)}$ for some constant $K$ and if $\max(p,q) \le 1-36 \log n / n$, then for any fixed $\|w\|=1$ graph attention fails to perfectly classify the nodes with probability at least $1-2\exp(-c'(1-\max(p,q))\log n)$ for some constant $c'$.
\end{enumerate}
\end{theorem}

\textit{Proof sketch}. For part 1 we use the attention coefficients from~\cref{prop:gammas} and we plug them in the definition of graph attention. Then because the distance between the means is large we can use simple concentration arguments to show that the convolved data concentrate around their means, which are classifiable with high probability using the classifier $w:=\sign(p-q)\mu/\|\mu\|$. This classifier measures correlation with one of the means. For part 2, the Gaussian noise dominates the means of the convolved node features. Thus there exists at least one node for which is not possible to detect its correct class with the given probability. 

\textit{Discussion of~\cref{thm:clean}}. There is a difference between the threshold in the positive result (part 1 of~\cref{thm:clean}) and the negative result (part 2 of~\cref{thm:clean}). The difference is prominent when $\max(p,q)=1-o_n(1)$. In that case, the threshold for the negative regime is very small and the probability can be so low that the result is not meaningful. This is an expected outcome. Consider the case $p=1$ and $q$ very small, then after convolution the data collapse approximately to two points that can be easily separated with a linear classifier. The difference between the two thresholds is small when $\max(p,q) \le 1 - \epsilon$ for any constant $\epsilon\in(0,1)$. That is, when $\max(p,q)$ is away from $1$. Finally, there is a difference due to $\omega$ in the positive result. This difference is not important since we can make the order in $\omega$ small with the cost of affecting the probability of perfect classification, although the probability will still be $1-o_n(1)$.

A limitation of our analysis is the assumption of a fixed $w$. Although in the proof of part 1 we utilize a specific fixed $w$ and we show that for this fixed $w$ graph attention is able to perfectly classify the nodes, it would be an interesting future work to set $w$ to be the optimal solution of some expected loss function.

It is important to note that if the edge features are clean then graph attention is better than graph convolution. In~\cref{thm:noisy} we will see that the threshold for graph convolution for the perfect classification is $\|\mu\| = \omega\left(\frac{p+q}{|p-q|}\sqrt{\frac{\log n}{n\max(p,q)}}\right)$, while for failing perfect classification the threshold is $\|\mu\| \le K\sigma \frac{p+q}{|p-q|}\sqrt{\frac{\log n}{n\max(p,q)}(1-\max(p,q))}$ for some constant $K$. By simply comparing these thresholds to those of graph attention in~\cref{thm:clean} it is easy to see that the parameter regime of CSBM where graph attention can perfectly classify the data is larger than that of graph convolution. First the difference $|p-q|$ is not affecting graph attention and also the threshold on the distance between the means is smaller for graph attention.

\subsection{Noisy edge features}\label{ssec:noisy-edge-features}

Consider the case of noisy edge features, this means that $\|\nu\|\le K \zeta$ for some constant $K$. We call this regime noisy because in this case there is a lot of overlap between the two Gaussians in the Gaussian mixture model of the CSBM. Note that there is a gap between this regime and the clean features regime, the two regimes differ by a factor of $\omega(\sqrt{\log (|\mathcal{E}|)})$. Although this factor grows with $|\mathcal{E}|$ we note that the factor changes very slowly with $|\mathcal{E}|$. For example, for $|\mathcal{E}|=10^{20}$, $\sqrt{\log (|\mathcal{E}|)}\approx 6.78$. Below we present the result about $\gamma$, which is crucial for obtaining node classification results and whose proof is deferred to the Appendix.

\begin{proposition}\label{prop:gammas_noisy}
Assume that $\|\nu\| \le K \zeta$ for some constant $K$. Then, with probability at least $1-o_n(1)$, there exists a subset of nodes $\mathcal{A} \subseteq [n]$ with cardinality at least $n - o(n)$ such that for all $i \in \mathcal{A}$ the following hold:
\begin{enumerate}
	\item There is a subset $J_{i,0} \subseteq N_i \cap C_0$ with cardinality at least $\frac{9}{10}|N_i \cap C_0|$, such that $\gamma_{ij} = \Theta(1/|N_i|)$ for all $j \in J_{i,0}$;
	\item There is a subset $J_{i,1} \subseteq N_i \cap C_1$ with cardinality at least $\frac{9}{10}|N_i \cap C_1|$, such that $\gamma_{ij} = \Theta(1/|N_i|)$ for all $j \in J_{i,1}$.
\end{enumerate}
\end{proposition}


The above proposition states that for the majority of the nodes at least $90\%$ of their intra- and inter-edge attention coefficients are up to a constant uniform. This behaviour is similar to that of GCN. We utilize~\cref{prop:gammas_noisy} in the next theorem to prove that graph attention performs similarly to graph convolution.

\begin{theorem}\label{thm:noisy}
Let $(A,X, E)\sim CSBM$, and assume that $\|\nu\| \le K \zeta$ for some constant $K$.
\begin{enumerate}
    \item If $\|\mu\| = \omega\left(\sigma\frac{p+q}{|p-q|}\sqrt{\frac{\log n}{n\max(p,q)}}\right)$, with probability $1-o_n(1)$ graph attention classifies all nodes correctly.
    \item If $\|\mu\| \le K' \sigma \frac{p+q}{|p-q|}\sqrt{\frac{\log n}{n\max(p,q)}(1-\max(p,q))}$ for some constant $K'$  and if $\max(p,q) \le 1-36 \log n / n$, then for any fixed $\|w\| = 1$ graph attention fails to perfectly classify the nodes with probability at least $1-2\exp(-c'(1-\max(p,q))\log n)$ for some constant $c'$.
\end{enumerate}
\end{theorem}

\textit{Proof sketch}. 
The sketch of the proof of this theorem is similar to~\cref{thm:clean}. The major difference is that when the edge features are noisy the majority of the attention coefficients are up to a constant uniform, see~\cref{prop:gammas_noisy}, and the attention mechanism is not able to distinguish intra- from inter-edges. Let's start the sketch for part 2. Using~\cref{prop:gammas_noisy} we prove that the convolved means of the node features get closer by $(p+q)/(p-q)$. That's how this quantity appears in the thresholds. Again, using concentration arguments for the noise in the data and the assumed bound on the distance between the means we can show that the noise is larger than the convolved means with high probability. Therefore, graph attention misclassifies at least one node with high probability. The proof for part 1 is easy, we simply pick a function $\Psi$ that allows us  to match the threshold in part 2. This is achieved by simply setting $\Psi=0$, which also happens to reduce graph attention to graph convolution since all attention coefficients are exactly uniform. Because the attention coefficients are uniform the term $n(p+q)$ appears in the threshold. The remaining of the proof follows the same approach as part 1 of~\cref{thm:clean}.


\textit{Discussion on~\cref{thm:noisy}}. Note that the same thresholds hold for graph convolution\footnote{The analysis for graph convolution is nearly identical to~\cref{thm:noisy}. Graph convolution has also been analyzed in~\cite{BFJ2021}, but our analysis includes the dependence on $p,q$, while in~\cite{BFJ2021} the authors consider $(p-q)/(p+q)$ to be a constant. Moreover, the analysis in our paper holds regardless of $p>q$ or $q>p$.}. 
Therefore, we conclude that graph attention is not better than graph convolution for perfect node classification when the edge features are noisy. 


\section{Synthetic Experiments}

We investigate empirically our theoretical results on the CSBM. We do $50$ trials and we present averaged results and standard deviation. In each trial we generate data $(A,X,E)$ that follow the CSBM. We use the constructed solutions that are described in our theorems, which come with the corresponding guarantees, to define the learnable parameters in the models. In particular, for the graph attention we set $w=\sign(p-q)\mu/\|\mu\|$. The attention function $\Psi$ is set to $\Psi(E_{(i,j)}):= s^T E_{(i,j)}$, where $s:=\sign(p-q)\nu/\|\nu\|$. For GCN we also set $w=\sign(p-q)\mu/\|\mu\|$. Furthermore, we set $n=400$, $d=n/\log^2 n$, $\sigma=0.1$ and each class has $n/2$ nodes.

\subsection{Clean edge features}

We set $\zeta=0.1$ and we pick the mean of the edge features $\nu$ such that $\|\nu\|\ge 100 \zeta \sqrt{\log (0.5n^2(p+q))}$.

\textbf{Varying $q$.} We perform two experiments to demonstrate parts 1 and 2 of~\cref{thm:clean}, For part 1, which is the positive result, we pick $\mu$ such that $\|\mu\| = 5\sigma \sqrt{\log n / (n\max(p,q))}$. For part 2, which is the negative result, we set $\|\mu\| = \sigma \sqrt{\log n / (n\max(p,q))}$. We fix $p=0.4$ and we vary $q$ from $\log^2/n$ to $2p$. In~\cref{fig:clean_edges_varying_q_positive} we present the positive result for graph attention and we also compare graph attention to graph convolution. For any value of $q$ graph attention achieves perfect classification while graph convolution depends on $|p-q|$. In~\cref{fig:clean_edges_varying_q_negative} we present the negative result for graph attention. In this experiment the distance between the means is small and graph attention fails to achieve perfect classification, but it is better than graph convolution.
 \begin{figure}[ht!]
      \centering
      \begin{subfigure}[b]{0.49\columnwidth}
          \centering
          \includegraphics[width=\columnwidth]{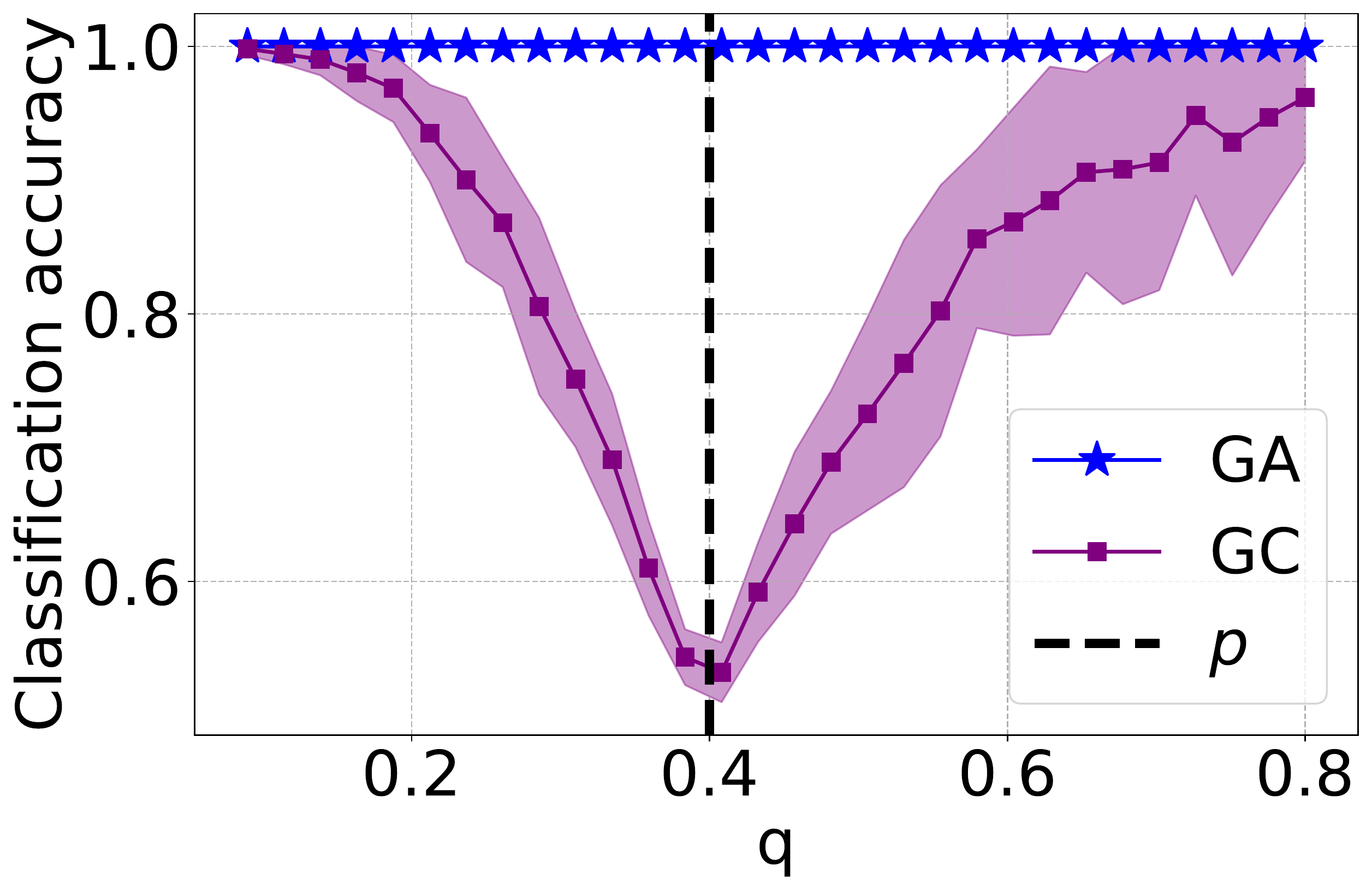}
          \caption{Positive result, part 1~\cref{thm:clean}}
          \label{fig:clean_edges_varying_q_positive}
      \end{subfigure}
      \begin{subfigure}[b]{0.49\columnwidth}
          \centering
          \includegraphics[width=\columnwidth]{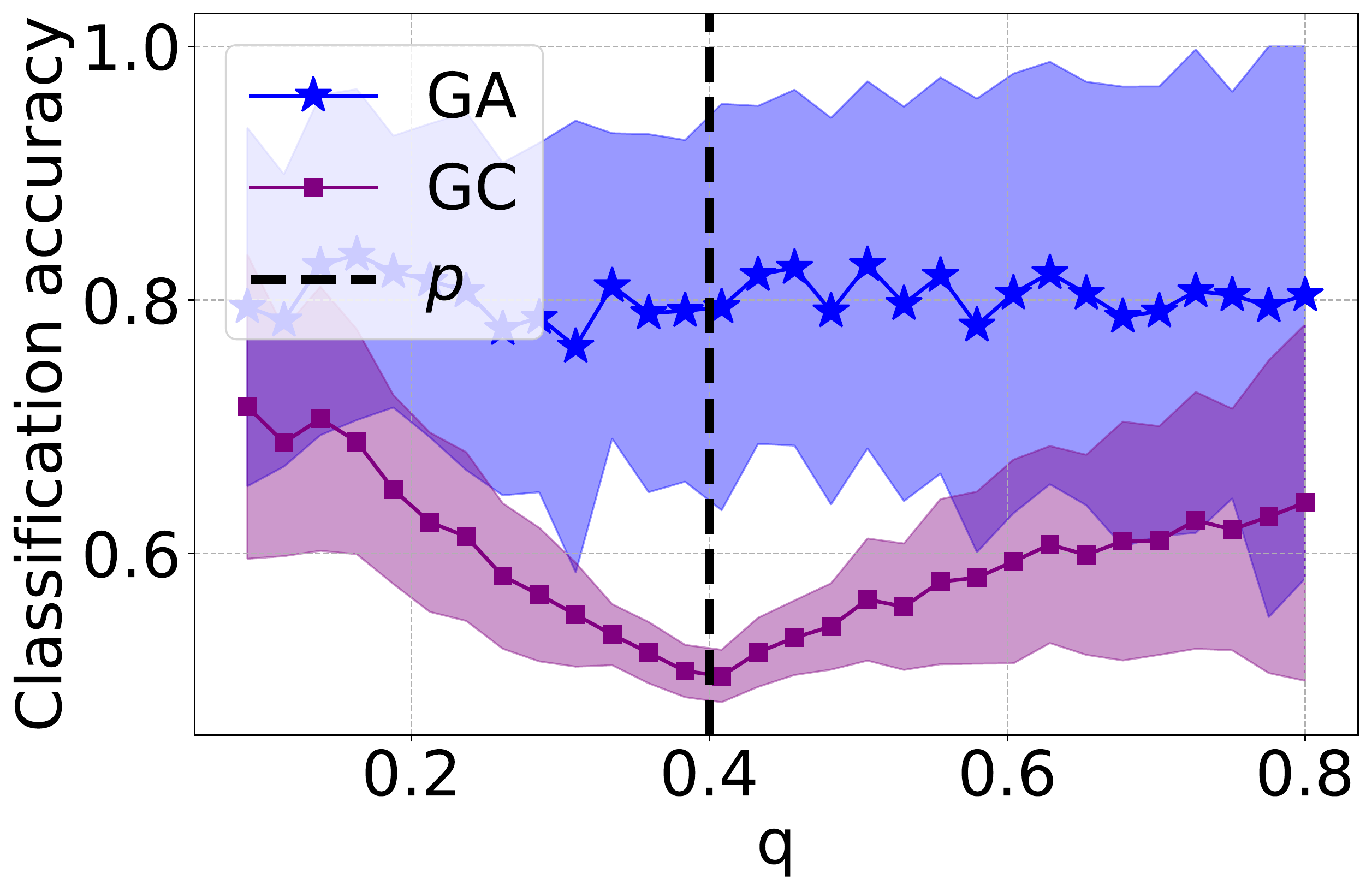}
          \caption{Negative result, part 2~\cref{thm:clean}}
          \label{fig:clean_edges_varying_q_negative}
      \end{subfigure}
      \caption{Demonstration of part 1 and 2 of~\cref{thm:clean}. }\label{fig:clean_edges_varying_q}
 \end{figure}

\textbf{Varying the distance of the means of the node features.} We illustrate how the attention coefficients and the accuracy change with the distance $\|\mu\|$. We fix $p=0.4$ and $q=0.33$ (this makes the graphs sufficiently noisy). We vary $\|\mu\|$ from $0.1\sigma \sqrt{\log n / (n\max(p,q))}$ to $20\sigma \sqrt{\log n}$. \cref{fig:clean_edges_varying_distance_gammas} illustrates~\cref{prop:gammas}. We observe empirically the separation of intra- and inter-$\gamma$ as claimed in~\cref{prop:gammas}. \cref{fig:clean_edges_varying_distance_accuracy} illustrates a combination of parts 1 and 2 of~\cref{thm:clean}. We observe that when the edge features are clean and $|p-q|$ is small then graph attention is better than graph convolution.
 \begin{figure}[ht!]
      \centering
      \begin{subfigure}[b]{0.49\columnwidth}
          \centering
          \includegraphics[width=\columnwidth]{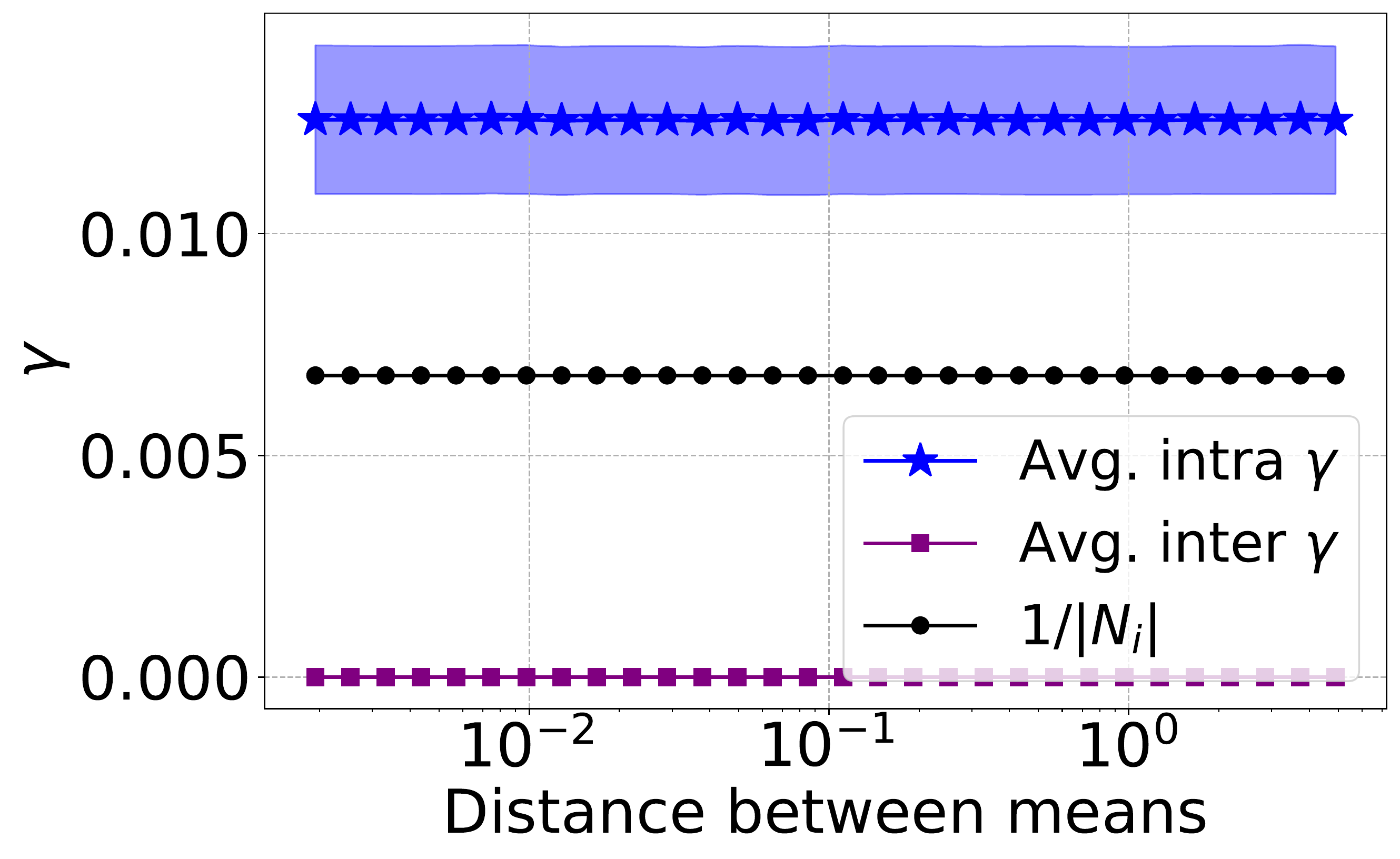}
          \caption{\cref{prop:gammas}}
          \label{fig:clean_edges_varying_distance_gammas}
      \end{subfigure}
      \begin{subfigure}[b]{0.49\columnwidth}
          \centering
          \includegraphics[width=\columnwidth]{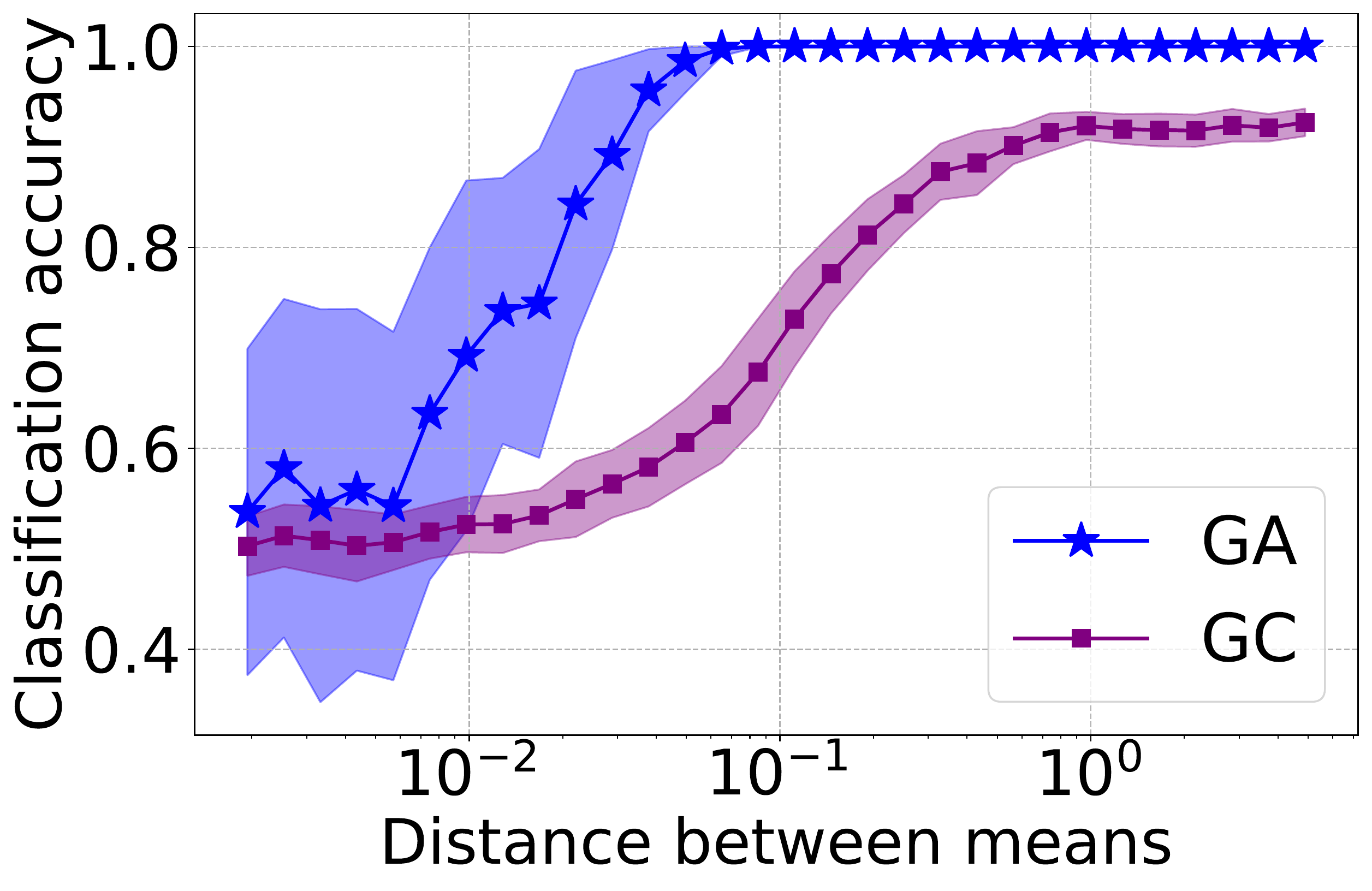}
          \caption{Parts 1 and 2 of~\cref{thm:clean}}
          \label{fig:clean_edges_varying_distance_accuracy}
      \end{subfigure}
      \caption{This figure demonstrates how the attention coefficients and the accuracy are changing as a function of the distance between the means of the node features. }\label{fig:clean_edges_varying_distance}
 \end{figure}
 
\subsection{Noisy edge features}

We set $\zeta=0.1$ and we pick the mean of the edge features $\nu$ such that $\|\nu\|\ge 100 \zeta$.

\textbf{Varying $q$.} We fix $p=0.4$ and we vary $q$ from $\log^2/n$ to $2p$. We perform two experiments to demonstrate parts 1 and 2 of~\cref{thm:noisy}, For part 1, which is the positive result, we pick $\mu$ such that $\|\mu\| = 8\sigma \frac{p+q}{|p-q|}\sqrt{\log n / (n\max(p,q))}$. For part 2, which is the negative result, we set $\|\mu\| = 0.1\sigma \frac{p+q}{|p-q|}\sqrt{\log n / (n\max(p,q))}$. In~\cref{fig:noisy_edges_varying_q_positive} we present the positive result for graph attention and we also compare graph attention to graph convolution. We observe that for any value of $q$ graph attention has very similar performance to graph convolution. In~\cref{fig:noisy_edges_varying_q_negative} we present the negative result for graph attention. In this experiment, graph attention has very similar performance to graph convolution because the attention coefficients are approximately uniform.
 \begin{figure}[ht!]
      \centering
      \begin{subfigure}[b]{0.49\columnwidth}
          \centering
          \includegraphics[width=\columnwidth]{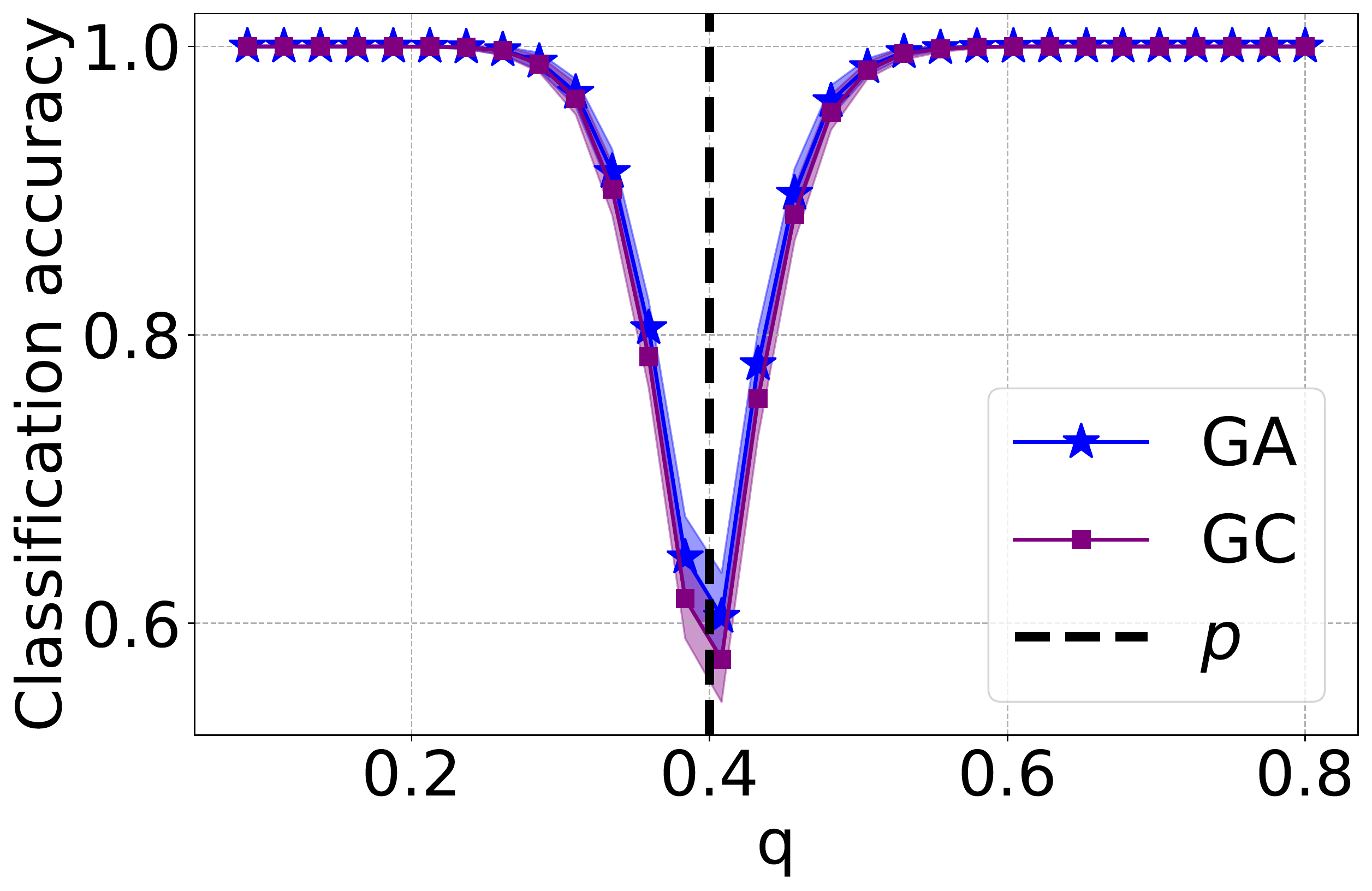}
          \caption{Positive result, part 1~\cref{thm:noisy}}
          \label{fig:noisy_edges_varying_q_positive}
      \end{subfigure}
      \begin{subfigure}[b]{0.49\columnwidth}
          \centering
          \includegraphics[width=\columnwidth]{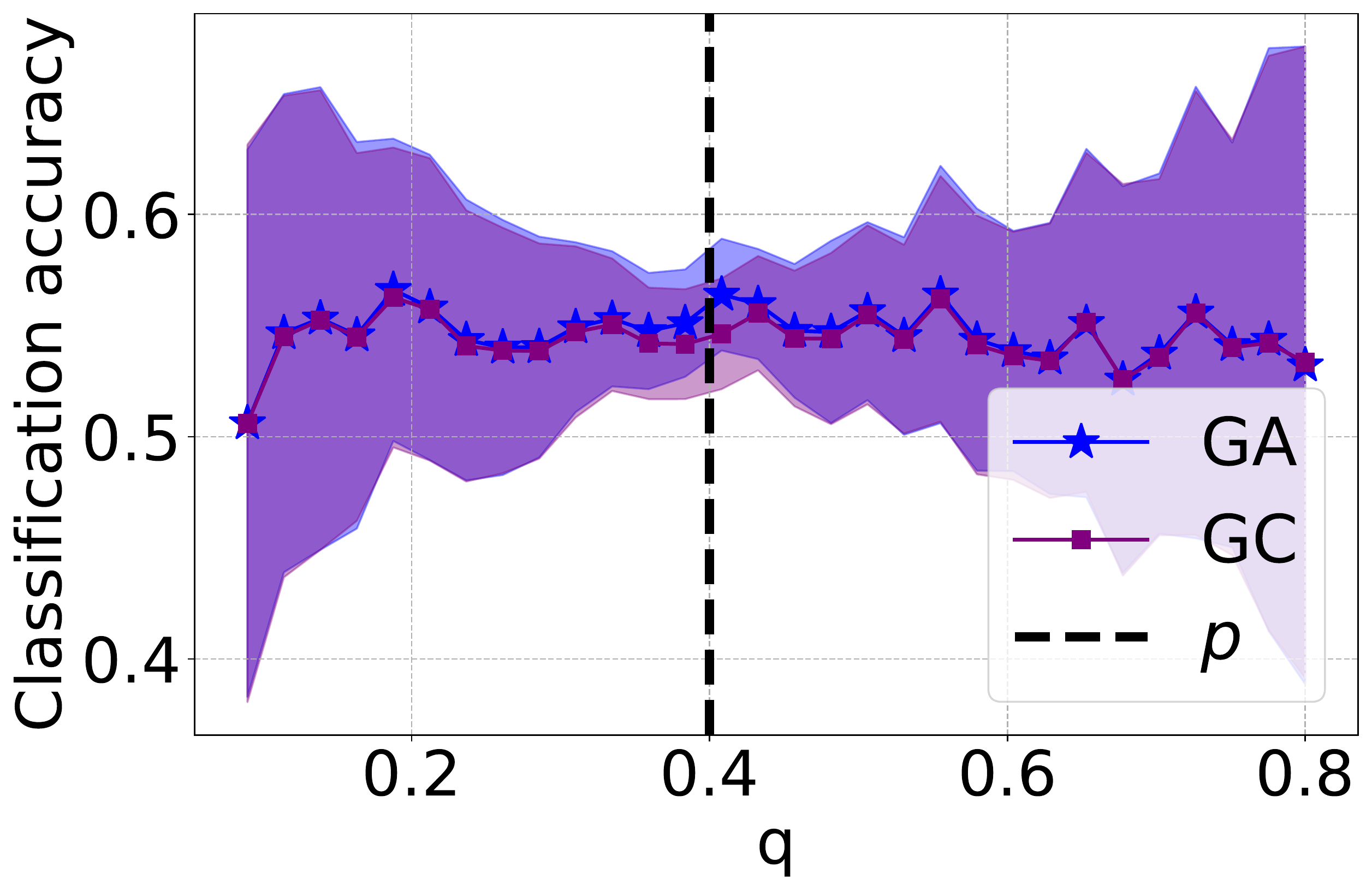}
          \caption{Negative result, part 2~\cref{thm:noisy}}
          \label{fig:noisy_edges_varying_q_negative}
      \end{subfigure}
      \caption{Demonstration of part 1 and 2 of~\cref{thm:noisy}. }\label{fig:noisy_edges_varying_q}
 \end{figure}

\textbf{Varying the distance of the means of the node features.} In this experiment we illustrate how the attention coefficients and the accuracy change as a function of the distance $\|\mu\|$. We fix $p=0.4$ and $q=0.33$ (this makes the graphs sufficiently noisy). We vary $\|\mu\|$ from $0.01\sigma \frac{p+q}{|p-q|}\sqrt{\log n / (n\max(p,q))}$ to $20\sigma \sqrt{\log n}$. \cref{fig:noisy_edges_varying_distance_gammas} illustrates~\cref{prop:gammas_noisy}. We observe empirically that intra- and inter-$\gamma$ concentrate around the same value and they are both approximately uniform as claimed in~\cref{prop:gammas_noisy}. \cref{fig:noisy_edges_varying_distance_accuracy} illustrates a combination of parts 1 and 2 of~\cref{thm:noisy}. We observe that graph attention has very similar performance to graph convolution.
 \begin{figure}[ht!]
      \centering
      \begin{subfigure}[b]{0.49\columnwidth}
          \centering
          \includegraphics[width=\columnwidth]{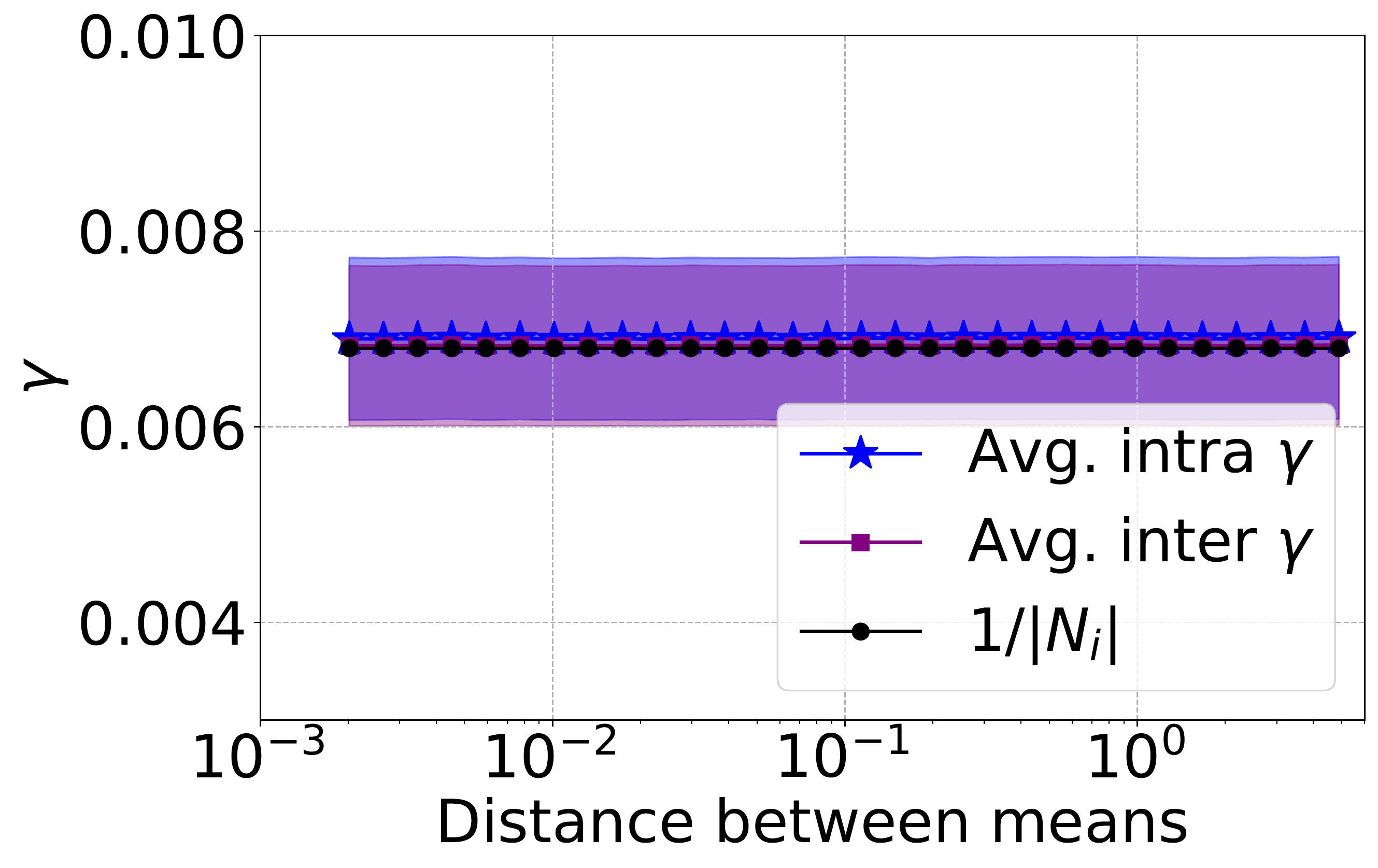}
          \caption{\cref{prop:gammas_noisy}}
          \label{fig:noisy_edges_varying_distance_gammas}
      \end{subfigure}
      \begin{subfigure}[b]{0.49\columnwidth}
          \centering
          \includegraphics[width=\columnwidth]{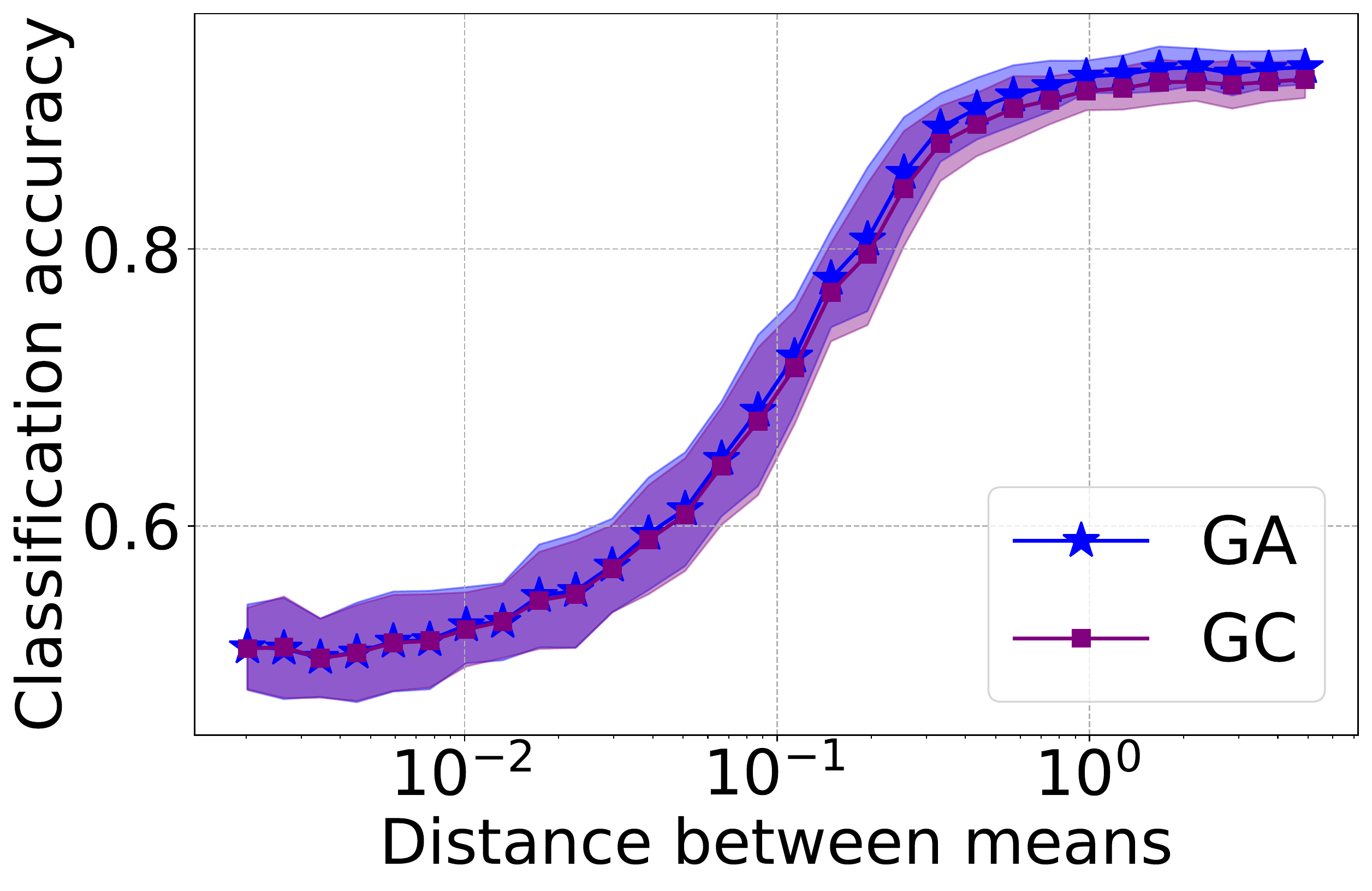}
          \caption{Parts 1 and 2 of~\cref{thm:noisy}}
          \label{fig:noisy_edges_varying_distance_accuracy}
      \end{subfigure}
      \caption{This figure demonstrates how the attention coefficients and the accuracy are changing as a function of the distance between the means of the node features. }\label{fig:noisy_edges_varying_distance}
 \end{figure}

\textbf{Attention coefficients for varying the distance of the means of the edge features.} In this experiment we demonstrate how the attention coefficients scale as a function of the distance between the means of the edge features. This is basically a combination of the results in~\cref{prop:gammas} and~\cref{prop:gammas_noisy}. We fix $p=0.4$ and $q=0.33$ (this makes the graphs sufficiently noisy). We vary $\|\nu\|$ from $0.01\zeta \sqrt{\log (0.5n^2(p+q))}$ to $30\zeta \sqrt{\log (0.5n^2(p+q))}$. The results are presented in~\cref{fig:varying_distance_of_edge_features_gammas}. We observe that for small distance between the means of the edge features the attention coefficients concentrate around the uniform measure, see~\cref{prop:gammas_noisy}, while as the distance increases then the intra-$\gamma$ increase up to the value $2/np$, see~\cref{prop:gammas} and the inter-$\gamma$ become very small, see~\cref{prop:gammas}. 

 \begin{figure}[ht!]
      \centering
      \begin{subfigure}[b]{0.5\columnwidth}
          \centering
          \includegraphics[width=\columnwidth]{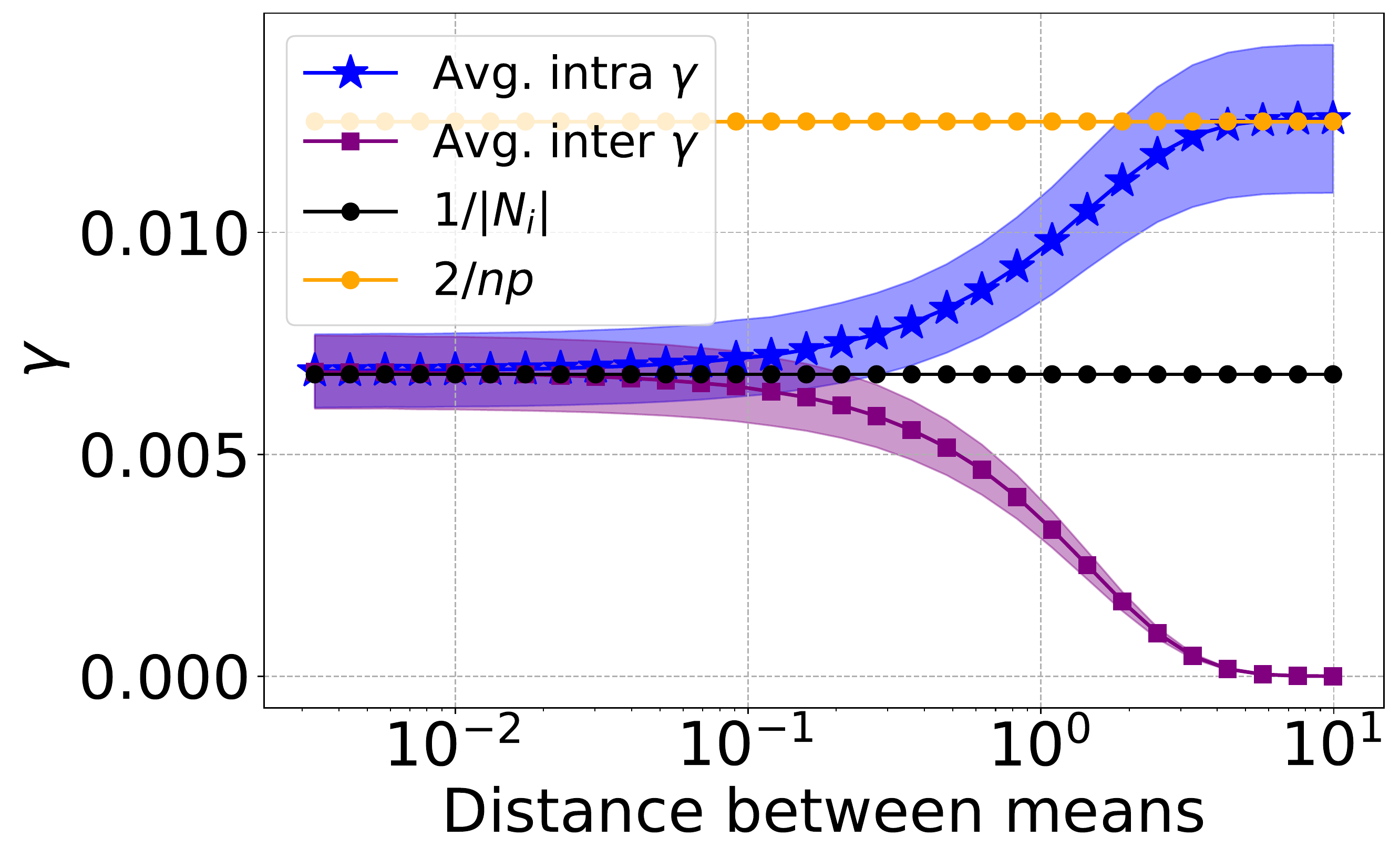}
      \end{subfigure}
      \caption{\cref{prop:gammas} and~\cref{prop:gammas_noisy}. }\label{fig:varying_distance_of_edge_features_gammas}
 \end{figure}
 
\section{Experiments on Real Data}


We use the popular real data Amazon Computers, Amazon Photos, Cora, PubMed, and CiteSeer. These data are publicly available and can be downloaded from~\cite{FL2019}. The datasets come with multiple classes, however, for each of our experiments we do a one-v.s.-all  classification. This is a semi-supervised problem, only a fraction of the training nodes have labels. The rest of the nodes are used for measuring prediction accuracy. For Cora, PubMed and CiteSeer we use the train/test split that is given by~\cite{FL2019}. For Amazon Computers and Photos where the train/test split is not given we sample randomly $1\%$ of the nodes for training, the rest are used for testing. For each dataset we split the number of features into the first half and second half. The former is used for node features and the latter is used for edge features. The edge features are given by the concatenation of the features adjacent to the edge.\footnote{In the appendix we also show experiments without the feature split where all features are used as node features. In this case we use the concatenated edge features as described in the original paper~\cite{Velickovic2018GraphAN} and also in our Preliminaries section. We observe similar performance as in edge splitting.} We present results averaged over $5$ trials to account for randomness in initialization of parameters.

We observe that graph attention is giving similar attention mass to intra- and inter-edges as graph convolution which uses uniform weights as attention coefficients. This also explains the that there is no clear winner between graph attention and graph convolution when it comes to performance. In~\cref{table:1} we illustrate these observations. We present results for class $0$ and $1$ of each dataset. The experiments for the other classes are shown in the appendix. In~\cref{table:1} the intra-mass column is the percentage of the total probability mass assigned to intra-edges by intra-edge attention coefficients. Similarly for the inter-mass column. We observe that graph attention and graph convolution assign similar percentage of the mass to intra- and inter-edges. This results in graph attention performing similarly to graph convolution. We observe the same results for the rest of the classes in the appendix. It is important to mention that the attention coefficients of graph attention might not be exactly uniform or up to a constant uniform since our assumptions for CSBM might be violated, however, we still observe that graph attention has overall the same allocation of intra- and inter-mass to graph convolution. Finally, we observe that the majority of mass is assigned to intra-edges. This is expected since we are solving one-v.s.-all classification problems. However, we would still expect graph attention to have much better mass allocation for inter-edges than graph convolution, but it does not.

\begin{table}[ht!]
\caption{Percentages of intra- and inter-mass allocation for graph attention (GA) and graph convolution (GC), and test accuracy.} \vspace{0.1cm}
\centering
 \begin{tabular}{||c | c c c c c ||} 
 \hline
 data & class & method & intra-m & inter-m  & acc.\\ [0.5ex] 
 \hline\hline
\multirow{4}{*}{\rotatebox[origin=c]{90}{Amzn Co.}} & \multirow{2}{*}{$0$} & GC & $98.7$ & $1.3$ & $96.8$ \\
& & GA & $98.1$ & $1.9$ & $96.7$ \\ \cline{2-6}
& \multirow{2}{*}{$1$} & GC & $93.6$ & $6.4$ & $91.4$ \\
& & GA & $93.3$ & $6.7$ & $88.7$ \\
\hline \hline
\multirow{4}{*}{\rotatebox[origin=c]{90}{Amzn Ph.}} & \multirow{2}{*}{$0$} & GC & $99.0$ & $1.0$ & $95.1$ \\
& & GA & $98.8$ & $1.2$ & $96.0$ \\ \cline{2-6}
& \multirow{2}{*}{$1$} & GC & $94.9$ & $5.1$ & $94.2$ \\
& & GA & $95.5$ & $4.5$ & $89.0$ \\
\hline \hline
\multirow{4}{*}{\rotatebox[origin=c]{90}{Cora}} & \multirow{2}{*}{$0$} & GC & $94.2$ & $5.8$ & $89.5$ \\
& & GA & $94.7$ & $5.3$ & $88.6$ \\ \cline{2-6}
& \multirow{2}{*}{$1$} & GC & $97.3$ & $2.7$ & $92.3$ \\
& & GA & $97.4$ & $2.6$ & $93.7$ \\
\hline \hline
\multirow{4}{*}{\rotatebox[origin=c]{90}{PubMed}} & \multirow{2}{*}{$0$} & GC & $92.2$ & $7.8$ & $82.1$ \\
& & GA & $92.1$ & $7.9$ & $82.8$ \\ \cline{2-6}
& \multirow{2}{*}{$1$} & GC & $91.4$ & $8.6$ & $58.8$ \\
& & GA & $90.6$ & $9.4$ & $59.5$ \\
\hline \hline
\multirow{4}{*}{\rotatebox[origin=c]{90}{CiteSeer}} & \multirow{2}{*}{$0$} & GC & $94.3$ & $5.7$ & $92.0$ \\
& & GA & $94.4$ & $5.6$ & $91.6$ \\ \cline{2-6}
& \multirow{2}{*}{$1$} & GC & $92.6$ & $7.4$ & $82.6$ \\
& & GA & $92.8$ & $7.2$ & $82.8$ \\
\hline 
\end{tabular}
\label{table:1}
\end{table}
\section{Conclusion}
We study conditions on the parameter of the CSBM with edge features such that graph attention can achieve or fail perfect node classification. We split our results into two parts. The first part is when the edge features are clean and the second part is when the edge features are noisy. If the edge features are clean we show that graph attention is able to distinguish intra from inter attention coefficients which allows us to prove that the condition for perfect classification is better than that of graph convolution. If the edge features are noisy we show that the majority of attention coefficients are up to a constant uniform which then implies that graph attention performs similarly to graph convolution.

Working with synthetic data models has a lot of limitations due to their gap with real data, but they also have very important advantages such as providing a solid insight about the performance of methods. It is more productive to discuss limitations in our analysis for potential future researchers who might wish to extend the present work. A limitation of our analysis is the assumption of a fixed $w$ for our negative results. It would be interesting future work to set $w$ to be the optimal solution of some expected loss function. Finally, it would be interesting to study the performance of methods beyond perfect classification.

\bibliography{references}
\bibliographystyle{plain}

\clearpage
\appendix 
\section{Elementary results}
Since $(\epsilon_i)_{i\in[n]}\sim\mbox{Ber}(\frac12)$, by the Chernoff bound~\cite[Section 2]{Vershynin:2018} we have that the number of nodes in each class satisfies
\begin{equation*}
    P\left(\frac{|C_0|}{n},\frac{|C_1|}{n}\in\left[\frac{1}{2}-o_n(1),\frac{1}{2}+o_n(1)\right]\right)\geq1-1/\poly(n).
\end{equation*}

\begin{proposition}[Concentration of degrees,~\cite{BFJ2021}]\label{prop:degree-conc}
Assume that the graph density is $p,q = \Omega\left(\frac{\log^2n}{n}\right)$. Then for any constant $c>0$, with probability at least $1-2n^{-c}$, we have for all $i\in[n]$ that
\begin{align*}
    |N_i| &= \frac n2(p+q)(1 \pm o_n(1))
\end{align*}
where the error term $o_n(1) = O\left({\sqrt{\frac{c}{\log n}}}\right)$.
\end{proposition}
\begin{proof}
Note that $|N_i|$ is a sum of $n$ Bernoulli random variables, hence, we have by the Chernoff bound~\cite[Section 2]{Vershynin:2018} that
\[
P\left(|N_i| \in \left[\frac n2(p+q)(1-\delta), \frac n2(p+q)(1+\delta)\right]^c \right)\leq 2\exp(-Cn(p+q)\delta^2),
\]
for some $C>0$. We now choose $\delta=\sqrt{\frac{(c+1)\log n}{Cn(p+q)}}$ for a large constant $c>0$. Note that since $p,q=\Omega\left(\frac{\log^2 n}{n}\right)$, we have that $\delta = O\left(\sqrt{\frac{c}{\log n}}\right) = o_n(1)$. Then following a union bound over $i\in[n]$, we obtain that with probability at least $1-2n^{-c}$,
\begin{align*}
|N_i| = \frac n2(p+q) \left(1 \pm O\Big(\sqrt{\frac{c}{\log n}}\Big)\right)\; \text{for all } i\in[n]
\end{align*}
\end{proof}

\begin{proposition}[Concentration of number of neighbors in each class]\label{prop:class_degree}
Assume that the graph density is $p,q = \Omega\left(\frac{\log^2 n}{n}\right)$. Then for any constant $c>0$, with probability at least $1 - 2n^{-c}$,
\begin{align*}
    |N_i\cap C_l| &= \frac n2 p(1\pm o_n(1)) \quad \text{for all } i\in C_l \mbox{ and } l\in\{0,1\} \\
    |N_i\cap C_l| &= \frac n2 q(1\pm o_n(1)) \quad \text{for all } i\in C_l^c \mbox{ and } l\in\{0,1\}
\end{align*}
where the error term $o_n(1) = O\left({\sqrt{\frac{c}{\log n}}}\right)$.
\end{proposition}
\begin{proof}
For any two distinct nodes $i,j\in[n]$ we have that $|N_i\cap C_l| = \sum_{j \in C_l}a_{ij}$. This is a sum of independent Bernoulli random variables, with mean $\E|N_i \cap C_l| = \frac{n}{2}p$ if $i\in C_l$ and $\E|N_i \cap C_l| = \frac{n}{2}q$ if $i\in C_l^c$. Denote $\mu_{ij}=\E|N_i\cap C_l|$. Therefore, by the Chernoff bound \cite[Section 2]{Vershynin:2018}, we have for a fixed pair of nodes $(i,j)$ that
\begin{align*}
P\left(|N_i\cap C_l| \in \left(\mu_{ij}(1-\delta_{ij}), \mu_{ij}(1+\delta_{ij})\right)^c\right)\leq 2\exp(-C\mu_{ij}\delta_{ij}^2)
\end{align*}
for some constant $C>0$. We now choose $\delta_{ij}=\sqrt{\frac{(c+2)\log n}{C\mu_{ij}}}$ for any large $c>0$. Note that since $p,q=\Omega\left(\frac{\log^2 n}{ n}\right)$, we have that $\delta_{ij} = O\left(\sqrt{\frac{c}{\log n}}\right) = o_n(1)$. Then following a union bound over all nodes $i\in [n]$, we obtain that with probability at least $1 - 2n^{-c}$, for all pairs of nodes $(i,j)$ we have
\begin{align*}
    |N_i\cap C_l| &= \frac n2 p(1\pm o_n(1)) \quad \text{for all } i\in C_l \mbox{ and } l\in\{0,1\} \\
    |N_i\cap C_l| &= \frac n2 q(1\pm o_n(1)) \quad \text{for all } i\in C_l^c \mbox{ and } l\in\{0,1\}.
\end{align*}
\end{proof}

\begin{proposition}[Concentration of uncommon neighbors]\label{prop:temp}
Assume that the graph density parameters satisfy $\max\{p,q\} \le 1-36 \log n /n$, then with probability at least $1 - n^{-1}$ we have that for all $i,j\in [n]$, $i \neq j$,
\begin{enumerate}
\item If $i,j \in C_l$ and $l \in \{0,1\}$, then
\begin{align*}
    |(N_i \cup N_j)\backslash (N_i\cap N_j)| 
    \ge \frac{n}{2}(p+q - p^2 -q^2)(1-o_n(1)),
\end{align*}
moreover,
\begin{align*}
    |((N_i \cup N_j)\backslash (N_i\cap N_j)) \cap C_l| &\ge \frac{n}{2}(p - p^2)(1-o_n(1)), \\
    |((N_i \cup N_j)\backslash (N_i\cap N_j)) \cap C_l^c| &\ge \frac{n}{2}(q - q^2)(1-o_n(1)).
\end{align*}
\item If $i \in C_l$, $j \in C_l^c$ and $l \in \{0,1\}$, then
\begin{align*}
    |(N_i \cup N_j)\backslash (N_i\cap N_j)| 
    \ge \frac{n}{2}(p+q - 2pq)(1-o_n(1)).
\end{align*}
\end{enumerate}
\end{proposition}

\begin{proof}
Consider an arbitrary pair of nodes $i,j\in C_0$ such that $i\neq j$. The probability that a node $l \in [n]$ is a neighbor of exactly one of $i,j$ is $2p(1-p)$ if $l\in C_0$ and $2q(1-q)$ if $l\in C_1$. Let $J_{ij}:=(N_i \cup N_j)\backslash (N_i\cap N_j)$. Then $J_{ij}$ is a sum of independent Bernoulli random variables and $\E|J_{ij}| = n(p(1-p) + q(1-q))$. Hence, by the multiplicative Chernoff bound we have that for any $0 < \delta \le 1$,
\[
    P\Big(|J_{ij}| \le \E|J_{ij}|(1-\delta)\Big) \le \exp\left(-\frac{1}{3} \delta^2\E|J_{ij}|  \right).
\]
In what follows we find a suitable choice for $\delta$. Because $\max\{p,q\} \le 1-36 \log n/n$, we have that
\[
    p+q-p^2-q^2 \ge \max\{p - p^2, q - q^2\} \ge 1- \frac{36 \log n}{n} - \left(1-\frac{36 \log n}{n}\right)^2 = \frac{36 \log n}{n}\left(1 - \frac{36 \log n}{n}\right)
\]
and hence
\[
     3\sqrt{\frac{\log n}{\E |J_{ij}|}} = 3\sqrt{\frac{\log n}{n(p+q-p^2-q^2)}} \le \frac{1}{2\sqrt{1 - \frac{36 \log n}{n}}} \le \frac{1}{2}\left(\frac{1}{1-6\sqrt{\frac{\log n}{n}}}\right) \le \frac{1}{2}\left(1 + O\left(\sqrt{\frac{\log n}{n}}\right)\right).
\]
Therefore we may choose $\delta = 3\sqrt{ \log n} /\sqrt{\E |J_{ij}|}$ and apply the union bound over all $i,j \in [n]$ to get that with probability at least $1 - n^{-1}$, for all $i,j \in C_0$, $i \neq j$,
\begin{align*}
    |J_{ij}| \ge \E |J_{ij}|(1-\delta) = n(p + q - p^2 - q^2)\left(1- \frac{1}{2} -O\left(\sqrt{\frac{\log n}{n}}\right)\right) = \frac{n}{2}(p + q - p^2 - q^2)(1-o_n(1)),
\end{align*}
which proves the claim on the cardinality of $|(N_i \cup N_j)\backslash (N_i\cap N_j)|$ for $i,j \in C_l$ and $l=0$. The other cases follow analogously.
\end{proof}

\section{Proofs for clean edge features}

Without loss of generality we ignore the self-loops in the graph. This is because adding self-loops only introduces an additional independent random variable, which changes the results up to an unimportant constant. Moreover, for proofs that rely on constructing the function $\Psi$ we provide the general definition for any $p,q$. However, since the proofs for $p>q$ and $q>p$ are almost identical, we provide the proofs for the case $p>q$. The proof for $q>p$ is different up to flipping signs and considering the fact that on expectation the inter-edges are more than the intra-edges.

\begin{proposition}\label{a-prop:gammas}
Let $(A,X, E)\sim CSBM(n,p,q,\mu,\nu,\sigma, \zeta)$, and assume that $\|\nu\|\ge \omega(\zeta \sqrt{\log (|\mathcal{E}|)})$. If $p>q$, we have that
    \begin{equation*}
    \gamma_{ij}=
        \begin{cases}
            \frac{2}{n p}(1\pm o_n(1)) & i,j\in C_0 \mbox{ or } i,j\in C_1\\
            o\left(\frac{1}{n(p+q)}\right) & \mbox{otherwise.}\\
        \end{cases}
    \end{equation*}
    with probability $1-o_{n}(1)$. If $q>p$, we have that
    \begin{equation*}
    \gamma_{ij}=
        \begin{cases}
            o\left(\frac{1}{n(p+q)}\right) & i,j\in C_0 \mbox{ or } i,j\in C_1\\
            \frac{2}{n q}(1\pm o_n(1)) & \mbox{otherwise.}\\
        \end{cases}
    \end{equation*}
    with probability $1-o_{n}(1)$.
\end{proposition}

\begin{proof}
    We will construct a $\Psi$ function such that with high probability it is able to separate intra- from inter-edges and it concentrates around its mean. Then we will use $\Psi$ to show that the attention coefficients $\gamma_{ij}$ concentrate as well. Define $s:=\sign(p-q)\nu/(\zeta\|\nu\|)$ and  $\Psi(E_{(i,j)}):= \alpha s^T E_{(i,j)}$ where $\alpha > 0$ is a scaling parameter whose value we will determine later. We will prove the result for the case that $p>q$, the result for $q > p$ is similar. We will show that function $\Psi$ concentrates around its mean. First, let us rewrite $E_{(i,j)}:= (2\epsilon_i-1)(2\epsilon_j-1)\nu + \zeta f_{(i,j)}$. Thus we have that $\Psi(E_{(i,j)})=\alpha (2\epsilon_i-1)(2\epsilon_j-1) \|\nu\|/\zeta +  \alpha \tilde{f}_{(i,j)}$ where $f_{(i,j)} := \zeta s^T f_{(i,j)}$. Because $\|\zeta s\| = 1$ we know that $\tilde{f}_{(i,j)} \sim N(0,1)$, and thus using upper bound on the Gaussian tail probability, e.g. see \cite[Section 2]{Vershynin:2018}, we have that 
    $$
    P\left(|\tilde{f}_{(i,j)} | \ge 10 c\sqrt{\log(|\mathcal{E}|)} \right) \le \exp(100 \log(|\mathcal{E}|)),
    $$
    for some absolute constant $c > 0$. Taking a union bound over all $(i,j)\in\mathcal{E}$ we have that 
    $$
    P\left(|\tilde{f}_{(i,j)} | < 10 \sqrt{\log(|\mathcal{E}|)} \ \forall (i,j)\in \mathcal{E}\right) \ge |\mathcal{E}|\exp(100 \log(|\mathcal{E}|)) = \frac{1}{|\mathcal{E}|^{99}} = o_{|\mathcal{E}|}(1)
    $$
    Let $E^*$ denote that event that $|\tilde{f}_{(i,j)}| < 10c\sqrt{\log(|\mathcal{E}|)}$ for all $(i,j) \in \mathcal{E}$. Then the above inequality says that the event $E^*$ happens with probability at least $1-o_{|\mathcal{E}|}(1)$. Let us assume that $E^*$ happens. Then we have that
    \begin{equation*}
      \Psi(E_{(i,j)}) =  
      \begin{cases}
           \alpha \|\nu\|/\zeta \pm O(\alpha \sqrt{\log(|\mathcal{E}|)})  & i,j\in C_0 \mbox{ or } i,j\in C_1\\
          -\alpha \|\nu\|/\zeta \pm O(\alpha \sqrt{\log(|\mathcal{E}|)})  & \text{otherwise.}
    \end{cases}
    \end{equation*}
    If $i,j\in C_0$, then plugging in $\Psi$ into the attention coefficients $\gamma_{ij}$ we get that 
    \begin{align*}
        \gamma_{ij} 
        &= \frac{\exp\left(\alpha \|\nu\|/\zeta \pm O(\alpha \sqrt{\log(|\mathcal{E}|)})\right)}{\sum_{\ell\in N_i\cap C_0}\exp\left(\alpha \|\nu\|/\zeta \pm O(\alpha \sqrt{\log(|\mathcal{E}|)})\right) + \sum_{\ell\in N_i\cap C_1}\exp\left(-\alpha \|\nu\|/\zeta \pm O(\alpha \sqrt{\log(|\mathcal{E}|)})\right)}\\
        &=\frac{1}{\sum_{\ell\in N_i\cap C_0}\exp\left(\pm O(\alpha \sqrt{\log(|\mathcal{E}|)})\right) + \sum_{\ell\in N_i\cap C_1}\exp\left(-2\alpha \|\nu\|/\zeta \pm O(\alpha \sqrt{\log(|\mathcal{E}|)})\right)}.
    \end{align*}
    Note that because $\|\nu\|/\zeta = \omega(\sqrt{\log(|\mathcal{E}|)})$ we can set $\alpha$ such that $\alpha \sqrt{\log(|\mathcal{E}|)} = o_n(1)$ and $\alpha \|\nu\|/\zeta = \omega_n(1)$. Therefore we get
    \begin{align*}
        \gamma_{ij} 
        &= \frac{1}{\sum_{\ell\in N_i\cap C_0}\exp(\pm o_n(1)) + \sum_{\ell\in N_i\cap C_1}\exp(-\omega_n(1))}
        = \frac{1}{|N_i \cap C_0|(1 \pm o_n(1)) + |N_i\cap C_1|o_n(1)}\\
        &= \frac{1}{|N_i \cap C_0|}(1 \pm o_n(1))
        = \frac{2}{np} (1 \pm o_n(1))
    \end{align*}
    where the last equality follows from~\cref{prop:class_degree}.

    Following a similar reasoning for the other edges we get that when the event $E^*$ happens,
    \begin{equation*}
    \gamma_{ij}=
        \begin{cases}
            \frac{2}{np}(1\pm o_n(1)) & i,j\in C_0 \mbox{ or } i,j\in C_1\\
            o\left(\frac{1}{n(p+q)}\right) & \mbox{otherwise.}\\
        \end{cases}
    \end{equation*}
    Noting that the event happens with probability at least $1-o_{n}(1)$ completes the proof.
\end{proof}

\begin{theorem}\label{a-thm:clean}
Let $(A,X, E)\sim CSBM(n,p,q,\mu,\nu,\sigma, \zeta)$, and assume that $\|\nu\|\ge \omega(\zeta \sqrt{\log (|\mathcal{E}|)})$.
\begin{enumerate}
    \item If $\|\mu\|\ge \omega\left(\sigma \sqrt{\frac{\log{n}}{n \max(p,q)}}\right)$, then we can construct a graph attention architecture that classifies the nodes perfectly with probability $1-o_n(1)$.
    \item If $\|\mu\|\le  K \sigma  \sqrt{\frac{\log{n}}{n\max(p,q)}\left(1-\max(p,q)\right)}$ for some constant $K$, then for any fixed $\|w\|=1$ graph attention fails to perfectly classify the nodes with probability 
    $$
    1-2\exp(-c'(1-\max(p,q))\log n)
    $$ 
    for some constant $c'$.
\end{enumerate}
\end{theorem}

\begin{proof}
    We start by proving part 1 of the theorem.
    Define $w := \sign(p-q)\mu/\|\mu\|$.
    We will prove the result for $p>q$ since the analysis for $q>p$ is similar.
    Write $x_i = (2\epsilon_i - 1) \mu + \sigma g_i$ where $g_i \sim N(0,I)$. Denote $\tilde{g}_j := w^Tg_j$ for $j \in [n]$. Because $\|w\|=1$ we have $\tilde{g}_j \sim N(0,1)$.
    We will use the attention coefficients from~\cref{a-prop:gammas} in the graph attention. Let $i\in C_0$. We have that
    \begin{align*}
        \hat x_i & = \sum_{j\in [n]} \tilde{A}_{ij}\gamma_{ij} w^T x_j \\ 
                 & = -\sum_{j\in N_i \cap C_0} \frac{2}{np}\|\mu\|(1\pm o_n(1)) + \sum_{j\in N_i \cap C_1} o\left(\frac{1}{n(p+q)}\right)\|\mu\| \\
                 & \ \ \  + \sigma \sum_{j\in N_i \cap C_0} \frac{2}{np} \tilde{g}_j (1\pm o_n(1)) + \sigma \sum_{j\in N_i \cap C_1} o\left(\frac{1}{n(p+q)}\right) \tilde{g}_j
    \end{align*}
    Let us first work with the sums for $\|\mu\|$. Using~\cref{prop:class_degree} we have that 
    $$
        \sum_{j\in N_i \cap C_0} \frac{2}{np}\|\mu\|(1\pm o_n(1)) = \|\mu\|(1\pm o_n(1)) 
    $$
    and
    $$
        \sum_{j\in N_i \cap C_1} o\left(\frac{1}{n(p+q)}\right)\|\mu\| = \|\mu\|o_n(1)(1\pm o_n(1))
    $$
    Putting the two sums for $\|\mu\|$ together we have that 
    $$
    -\sum_{j\in N_i \cap C_0} \frac{2}{np}\|\mu\|(1\pm o_n(1)) + \sum_{j\in N_i \cap C_1} o\left(\frac{1}{n(p+q)}\right)\|\mu\| = -\|\mu\|(1\pm o_n(1)).
    $$
    Let us now work with the sum of noise over $N_i\cap C_0$. This is a sum of $|N_i\cap C_0|$ standard normals. From Theorem 2.6.3. (General Hoeffding’s inequality) and from concentration of $|N_i\cap C_0|$ from~\cref{prop:class_degree} we have that 
    $$
    P\left( \left|\sum_{j\in N_i \cap C_0} \Theta\left(\frac{1}{np}\right) \tilde{g}_j\right| \ge \sqrt{\frac{10 C^2 \log{n}}{n \max(p,q) c}} \right) \le 2\exp\left( -10\log{n}\right),
    $$
    where $c$ is a constant, and $C$ is the sub-Gaussian constant for $\tilde{g}_j$. Taking a union bound over all $i\in C_0$, we have that with probability $1-o_n(1)$ we have that 
    $$\left|\sum_{j\in N_i \cap C_0} \Theta\left(\frac{1}{np}\right) \tilde{g}_j\right| < \sqrt{\frac{10 C^2 \log{n}}{n \max(p,q) c}}, \ \forall i\in C_0.
    $$
    Using similar concentration arguments we get that the second sum of the noise over $N_i \cap C_1$ is a smaller order term. Thus, since $\|\mu\|\ge \omega\left(\sigma \sqrt{\frac{\log{n}}{n \max(p,q)}}\right)$ we get 
    $$
        \hat x_i =-\|\mu\|(1\pm o_n(1)) + o(\|\mu\|).
    $$
    with probability $1-o_n(1)$. Therefore, with high probability nodes in $C_0$ are correctly classified. Using the same procedure for nodes in $C_1$ we get that these nodes are also classified correctly.
    
    Let us now proceed with the proof of part 2. We will prove the result for $p>q$, the proof for the $q>p$ is similar. We will prove that the probability of classifying all nodes correctly is very small. Let us start with the event of correct classification of all nodes. Let $w$ be any vector satisfying $\|w\| = 1$.
    Using the same sub-Gaussian concentration arguments as before and~\cref{a-prop:gammas} we get that with probability $1-o_n(1)$ the event for perfect classification is
    \begin{align*}
        -w^T \mu(1\pm o_n(1)) + \max_{i\in C_0} \sigma \sum_{j\in N_i}\gamma_{ij} w^T g_j < 0 \\
        w^T \mu(1\pm o_n(1)) + \min_{i\in C_1} \sigma \sum_{j\in N_i} \gamma_{ij} w^T g_j > 0,
    \end{align*}
    for nodes in $C_0$ and $C_1$, respectively. Let's bound the probability of correct classification for $C_0$, the result for $C_1$ is similar. 
    \begin{align*}
        P\left(-w^T \mu(1\pm o_n(1)) + \max_{i\in C_0} \sigma \sum_{j\in N_i}\gamma_{ij} w^T g_j < 0\right) & \\
        \le P\left(\max_{i\in C_0} \sigma \sum_{j\in N_i}\gamma_{ij} w^T g_j < \|\mu\|(1\pm o_n(1))\right) & \quad \mbox{using Cauchy-Schwartz and $\|w\|=1$} \\
        \le P\left(\max_{i\in C_0}  \sum_{j\in N_i}\gamma_{ij} w^T g_j <  K  \sqrt{\frac{\log{n}}{n\max(p,q)}(1-\max(p,q))} \right) & \\
        = P\left(\max_{i\in C_0}  \sum_{j\in N_i}\gamma_{ij} w^T g_j <  K  \sqrt{\frac{\log{n}}{np}\left(1-p\right)} \right) & 
    \end{align*}
    The remaining of the proof is similar to the proof in~\cite{BFJ2021}.
    We will utilize Sudakov's minoration inequality~\cite[Section 7.4]{Vershynin:2018} to obtain a lower bound on the expected supremum of the corresponding Gaussian process, and then use Borell's inequality~\cite[Section 2.1]{Adler:2007} to upper bound the probability.
    
    Let $Z_i:=\sum_{j\in N_i}\gamma_{ij} w^T g_j$ for $i\in C_0$. 
    To apply Sudakov's minoration result, we also define the canonical metric $d_T(i,j) =\sqrt{\E[(Z_i - Z_j)^2]}$ for any $i,j\in C_0$. In what follows we will first compute $\E[(Z_i - Z_j)^2]$ and then the metric. Conditioned on the events described by \cref{prop:degree-conc}, \cref{prop:class_degree} and \cref{prop:temp}, we know that with probability at least $1-o_n(1)$,
    \begin{align*}
        \E[(Z_i - Z_j)^2] 
        & = \E\left[\left(\sum_{l\in N_i}\gamma_{il} w^T g_l\right)^2 + \left(\sum_{k\in N_j}\gamma_{jk} w^T g_k\right)^2 - 2 \left(\sum_{l\in N_i}\gamma_{il} w^T g_l\right)\left(\sum_{k\in N_j}\gamma_{jk} w^T g_k\right)\right] \\ 
        & = \sum_{l\in N_i}\gamma_{il}^2 + \sum_{l\in N_j}\gamma_{jl}^2 - 2\sum_{l\in N_i \cap N_j} \gamma_{il}\gamma_{jl}\\
        & \ge \sum_{l \in (N_i \backslash (N_i \cap N_j))\cap C_0} \gamma_{il}^2 + \sum_{l \in (N_j \backslash (N_i \cap N_j))\cap C_0} \gamma_{jl}^2\\
        & = \sum_{l \in ((N_i \cup N_j)\backslash (N_i \cap N_j))\cap C_0} \frac{4}{n^2p^2} (1 \pm o_n(1))\\
        & = \frac{n}{2}(p-p^2)(1-o_n(1)) \cdot \frac{4}{n^2p^2} (1 \pm o_n(1)) \\
        & = \frac{2}{np}(1-p)(1-o_n(1)).
    \end{align*}
    Thus 
    $$
    d_T(i,j) \ge \sqrt{\frac{2}{np}}\sqrt{1-p}(1- o_n(1))
    $$
    Using this result in Sudakov's minoration inequality, we obtain that
    $$\E[\max_i Z_i] \ge c \sqrt{\frac{\log n}{np}(1-p)}.
    $$
    for some absolute constant $c$.
    We now use Borell's inequality~\cite[Section 2.1]{Adler:2007} and the fact that the variance of the Gaussian data after graph attention convolution is $\Theta(1/np)$ to obtain that for any $t>0$,
    \begin{align*}
        &P\left(\max_{i\in C_0} Z_i \le \E \max_{i\in C_0} Z_i - t\right) \le 2\exp(-Ct^2np)\\
        \implies \quad & P\left(\max_{i\in C_0} Z_i \le c \sqrt{\frac{\log n}{np}(1-p)} - t\right) \le 2\exp(-Ct^2 np)
    \end{align*}
    for some absolute constant $C > 0$. Now, for some small enough constant $K$ we may set $t$ such that
    $$
    t = c \sqrt{\frac{\log n}{np}(1-p)} - K  \sqrt{\frac{\log{n}}{np}(1-p)} = \Omega\left(\sqrt{\frac{\log n}{np}(1-p)}\right).
    $$
    Plugging this $t$ in the above probability we have that for some constant $c'>0$,
    \begin{align*}
        P\left(\max_{i\in C_0} Z_i \le K  \sqrt{\frac{\log{n}}{np}(1-p)}\right)\le 2\exp(-c'(1-p)\log n).
    \end{align*}
\end{proof}

\section{Proofs for noisy edge features}

For noisy edge features we have $\|\nu\| \le K \zeta $ for some $K=\mathcal{O}(1)$. We may write $E_{(i,j)} = (2\epsilon_i-1)(2\epsilon_j-1)\nu + \zeta f_{(i,j)}$ where $f_{(i,j)} \sim N(0,I)$. That is, $E_{(i,j)} = \nu + \zeta f_{(i,j)}$ if $(i,j)$ is an intra-edge and $E_{(i,j)} = -\nu + \zeta f_{(i,j)}$ if $(i,j)$ is an inter-edge. Recall that in this work we consider attention architecture $\Psi$ that is a composition of a Lipschitz function and a linear function. That is, for $(i,j) \in \mathcal{E}$ the attention coefficient is given as
\[
	\gamma_{ij} = \frac{\exp(\Psi(E_{(i,j)}))}{\sum_{l \in N_i} \exp(\Psi(E_{(i,l)}))} = \frac{\exp(\phi((2\epsilon_i-1)(2\epsilon_j-1)s^T\nu + \zeta s^Tf_{(i,j)}))}{\sum_{l \in N_i} \exp(\phi((2\epsilon_i-1)(2\epsilon_l-1)s^T\nu + \zeta s^Tf_{(i,l)}))}
\]
where $\phi: \mathbb{R} \rightarrow \mathbb{R}$ is Lipschitz continuous with Lipschitz constant $L$ and $|\phi(0)| \le R$ for some $R \ge 0$. Naturally, both $L$ and $R$ do not depend on $n$. The linear function has learnable parameters $s$. We assume that $\|s\|$ is bounded. Therefore, in subsequent analysis we also assume $\|s\| = 1$. This assumption is without loss of generality, because as long as $s$ is bounded and nonzero, one may always write $s = rs'$ for some $\|s'\| = 1$ and absolute constant $r > 0$. The additional constant $r$ does affect the computations we need for the proofs.

We start by defining a number of index sets which we will use extensively in the proofs. First let us define a subset of nodes whose incident edge features are ``nice'',
\begin{equation}
\mathcal{A} := \{i \in [n] : |s^Tf_{(i,j)}| \le 10\sqrt{\log(n(p+q))} \ \forall j \in N_i\}
\end{equation}
In addition, for $i \in [n]$ define the following sets
\begin{align*}
	J_{i,0} &:= \left\{ j \in N_i \cap C_0 : |s^T f_{(i,j)}| \le \sqrt{10} \right\}, \\
	J_{i,1} &:= \left\{ j \in N_i \cap C_1 : |s^T f_{(i,j)}| \le \sqrt{10} \right\}, \\
	B_{i,0}^t &:= \left\{ j \in N_i \cap C_0 : 2^{t-1} \le  |s^T f_{(i,j)}| \le 2^t\right\}, t = 1,2, \ldots, T, \\
	B_{i,1}^t &:= \left\{ j \in N_i \cap C_1 : 2^{t-1} \le  |s^T f_{(i,j)}| \le 2^t\right\}, t = 1,2, \ldots, T
\end{align*}
where $T = \left\lceil \log_2\left(10\sqrt{\log(n(p+q))}\right)\right\rceil$. Finally, for a pair of nodes $i,j \in [n]$ we define
\[
    \widehat{J}_{ij} := \left\{ l \in (N_i \cup N_j) \backslash  (N_i \cap N_j) : |s^T f_{(i,j)}| \le \sqrt{10} \right\}.
\]

Since the sets defined above depends on the random variable $s^Tf_{(i,j)}$, the cardinalities of the sets are also random variables. In the following we provide high probability bounds on the cardinalities of these sets.

\begin{claim}[Lower bound of $|\mathcal{A}|$]\label{a-claim:cal_A}
With probability at least $1 - o_n(1)$ we have $|\mathcal{A}| \ge n - O(n / \log n)$, and consequently $|\mathcal{A} \cap C_0| \ge |C_0| - O(n / \log n)$ and $|\mathcal{A} \cap C_1| \ge |C_1| - O(n / \log n)$.
\end{claim}

\begin{proof}
We start by providing an upper bound on the cardinality of the following set
\[
	\mathcal{A}_E := \left\{(i,j) \in \mathcal{E} : |s^Tf_{(i,j)}| \ge 10\sqrt{\log(n(p+q))} \right\}.
\]
Note that we may write $|\mathcal{A}_E|$ as
\[
	|\mathcal{A}_E| = \sum_{(i,j)\in\mathcal{E}}\mathbf{1}_{\left\{|s^Tf_{(i,j)}| \ge 10\sqrt{\log(n(p+q))}\right\}},
\]
and thus by the multiplicative Chernoff bound we get that for any $\delta > 0$,
\begin{equation}\label{eq:setA_chernoff}
	P\Big(|\mathcal{A}_E|   \ge |\mathcal{E}|b(1+\delta)\Big) \le \exp\left(-\tfrac{\delta^2}{2+\delta}|\mathcal{E}|b\right),
\end{equation}
where
\[
	b := P\left(|s^Tf_{(i,j)}| \ge 10\sqrt{\log(n(p+q))}\right).
\]
Moreover, from standard upper bound on Gaussian tail probability we know that $b < e^{-50\log(n(p+q))}$.
Let us set
\[
	\delta := \frac{1}{bn^{1/2}|\mathcal{E}|^{1/2}(p+q)}.
\]
Using \cref{prop:degree-conc} we know that with probability at least $1 - o_n(1)$ one has $|\mathcal{E}| = \frac{n^2}{2}(p+q)(1 \pm o_n(1))$, and hence it follows that,
\[
	\delta = \frac{\sqrt{2}(1 \pm o_n(1))}{b(n(p+q))^{3/2}} \ge \frac{\sqrt{2} (1 \pm o_n(1))}{(n(p+q))^{3/2}e^{-50\log(n(p+q))}} = \omega_n(1).
\]
This means that
\begin{equation}\label{eq:setA_chernoff_prob}
	\frac{\delta^2}{2+\delta}|\mathcal{E}|b \ge \Omega(\delta |\mathcal{E}|b) = \Omega\left(\frac{ |\mathcal{E}|b}{bn^{1/2}|\mathcal{E}|^{1/2}(p+q)}\right) = \Omega\left(\frac{|\mathcal{E}|^{1/2}}{n^{1/2}(p+q)}\right) = \Omega\left(\frac{n}{\sqrt{n(p+q)}}\right) \ge \Omega(\sqrt{n}).
\end{equation}
On the other hand,
\begin{equation}\label{eq:setA_chernoff_bd}
\begin{split}
	|\mathcal{E}|b(1+\delta) 
	&= |\mathcal{E}|b + \frac{ |\mathcal{E}|b}{bn^{1/2}|\mathcal{E}|^{1/2}(p+q)} \le  |\mathcal{E}| e^{-50\log(n(p+q))} + \frac{|\mathcal{E}|^{1/2}}{n^{1/2}(p+q)}\\
	&\le \frac{n^2}{2}(p+q)(1 \pm o_n(1))\frac{1}{(n(p+q))^{50}} + \frac{n}{\sqrt{2n(p+q)}}(1 \pm o_n(1))\\
	&= O\left(\frac{n}{(n(p+q))^{49}}\right) + O\left(\frac{n}{\sqrt{n(p+q)}}\right) \le O\left(\frac{n}{\log n}\right),
\end{split}
\end{equation}
where the last inequality follows from the assumption that $p,q = \Omega(\log^2n / n)$. Combining \eqref{eq:setA_chernoff}, \eqref{eq:setA_chernoff_prob} and \eqref{eq:setA_chernoff_bd}, with probability at least $1 - o_n(1)$ we have that
\[
	|\mathcal{A}_E| \le O(n / \log n).
\]
This means that for any subset $S \subseteq [n]$, e.g. we may take $S = [n]$, $S = C_0$ or $S = C_1$,
\[
	\left|\left\{i \in S : \exists j \in N_i \ \mbox{such that} \ |s^Tf_{(i,j)}| \ge 10\sqrt{\log(n(p+q))} \right\}\right| \le |\mathcal{A}_E| \le  O(n / \log n),
\]
which proves the claim.
\end{proof}

\begin{claim}[Lower bounds of $|J_{i,0}|$ and $|J_{i,1}|$~\cite{fountoulakis2022graph}]\label{claim:card_J}
With probability at least $1-o_n(1)$, we have that for all $i \in [n]$,
\[
	|J_{i,0}| \ge \frac{9}{10}|N_i \cap C_0| \ \ \mbox{and} \ \ |J_{i,1}| \ge \frac{9}{10}|N_i \cap C_1|.
\]
\end{claim}

\begin{proof}
We prove the result for $J_{i,0}$, the result for $J_{i,1}$ follows analogously. Consider an arbitrary $i \in [n]$. For each $j \in N_i \cap C_0$ we have
\[
    P(|s^T f_{(i,j)}| \ge \sqrt{10}) \le e^{-50},
\]
which follows from upper bound of the Gaussian tail, e.g., see Proposition 2.1.2 in \cite{Vershynin:2018}. Denote $J_{i,0}^c := (N_i \cap C_0)\backslash  J_{i,0}$. Then
\[
    \E\left[|J_{i,0}^c|\right] = \E\left[\sum_{j \in N_i \cap C_0} \mathbf{1}_{\left\{||s^T f_{(i,j)}| \ge \sqrt{10}|\right\}}\right] \le e^{-50}|N_i \cap C_0|.
\]
Apply Chernoff bound (see, e.g., Theorem 2.3.4 in \cite{vershynin2018high}) we have
\begin{align*}
    P\left(|J_{i,0}^c| \ge \frac{1}{10}|N_i \cap C_0| \right)
    &\le e^{- \E\left[|J_{i,0}^c|\right]}\left(\frac{e \E\left[|J_{i,0}^c|\right]}{|N_i \cap C_0|/10}\right)^{|N_i \cap C_0|/10}\\
    &\le \left(\frac{e e^{-50}|N_i \cap C_0| }{|N_i \cap C_0|/10}\right)^{|N_i \cap C_0|/10}\\
    &=\exp\left(-\left(\frac{1}{2}-\frac{\log 10}{10} - \frac{1}{10}\right)|N_i \cap C_0|\right)\\
    &\le \exp\left(-\frac{4}{25}|N_i \cap C_0|\right).
\end{align*}
Apply the union bound we get
\begin{align*}
    P\left(|J_{i,0}| \ge \frac{9}{10}|N_i \cap C_0|, \forall i \in [n] \right) 
    &\ge 1- \sum_{i \in [n]}\exp\left(-\frac{4}{25}|N_i \cap C_0|\right) \\
    &\ge (1-o_n(1))\left(1-n\exp\left(-\frac{4}{25} n\min\{p,q\}(1-o_n(1))\right)\right)\\
    &= 1 - o_n(1).
\end{align*}
\end{proof}

\begin{claim}[Lower bound of $ |\widehat{J}_{ij}|$]\label{claim:card_J2}
Assume that the graph density parameters satisfy $\max\{p,q\} \le 1 - 36\log n /n$, then with probability at least $1 - o_n(1)$ we have that for all $i,j \in C_l$ and $l \in \{0,1\}$, $i \neq j$,
\begin{align*}
    |\widehat{J}_{ij} \cap C_l| &\ge \frac{9}{10}|((N_i \cup N_j) \backslash  (N_i \cap N_j)) \cap C_l|,\\
    |\widehat{J}_{ij} \cap C_l^c| &\ge \frac{9}{10}|((N_i \cup N_j) \backslash  (N_i \cap N_j)) \cap C_l^c|.
\end{align*}
\end{claim}

\begin{proof}
We prove the result for $i,j \in C_l$ and $l=0$, the other cases follow analogously. Consider an arbitrary pair of nodes $i,j \in C_0$ and$i \neq j$. For each $k \in ((N_i \cup N_j) \backslash  (N_i \cap N_j)) \cap C_0$ we have
\[
    P(|s^T f_{(i,k)}| \ge \sqrt{10}) \le e^{-50}
\]
as in the proof of Claim~\ref{claim:card_J}. Moreover, by following the same reasoning as in the proof of Claim~\ref{claim:card_J}, we get that
\begin{align*}
    &P\left(|\widehat{J}_{ij} \cap C_0| \ge \frac{9}{10}|((N_i \cup N_j) \backslash (N_i \cap N_j)) \cap C_0|, \forall i,j \in C_0, i \neq j \right) \\
    \ge~& 1- \sum_{i,j \in C_0, i\neq j}\exp\left(-\frac{4}{25}|((N_i \cup N_j) \backslash (N_i \cap N_j)) \cap C_0|\right) \\
    \ge~& (1-o_n(1))\left(1 - n^2 \exp\left(-\frac{4}{25}\frac{n}{2}(p-p^2)(1-o_n(1))\right)\right)\\
    \ge~& (1-o_n(1)) \left(1 - \exp\left(-\frac{72}{25}\log n(1-o_n(1)) + 2 \log n\right)\right)\\
    =~&1-o_n(1).
\end{align*}
In the above, the second and the third inequalities follow from \cref{prop:temp} and the assumption that $p \le 1- 36\log n /n$.
\end{proof}

\begin{claim}[Upper bounds of $|B_{i,0}^t|$ and $|B_{i,1}^t|$~\cite{fountoulakis2022graph}]\label{claim:card_B}
With probability at least $1-o_n(1)$, we have that for all $i \in [n]$ and for all $t \in [T]$,
\[
|B_{i,0}^t| \le E[|B_{i,0}^t|] + \sqrt{T}|N_i \cap C_0|^{4/5} \ \ \mbox{and} \ \ |B_{i,1}^t| \le E[|B_{i,1}^t|] + \sqrt{T}|N_i \cap C_1|^{4/5}.
\]
\end{claim}

\begin{proof}
We prove the result for $B_{i,0}^t$, and the result for $B_{i,1}^t$ follows analogously. First fix $i \in [n]$ and $t \in [T]$. By the additive Chernoff bound we have
\[
	\Pr\left(|B_{i,0}^t| \ge \Ex [|B_{i,0}^t|] +  |N_i \cap C_0| \cdot \sqrt{T}|N_i \cap C_0|^{-\frac{1}{5}}\right) \le  e^{-2T|N_i \cap C_0|^{3/5}}.
\]
Taking a union bound over all $i \in [n]$ and $t \in [T]$ we get
\begin{align*}
	&\Pr\left[\bigcup_{i\in [n]}\bigcup_{t\in[T]} \left\{|B_{i,0}^t| \ge \Ex [|B_{i,0}^t|] + \sqrt{T}|N_i \cap C_0|^{\frac{4}{5}}\right\}\right]\\
	\le~& nT\exp\left(-2T\left(\frac{n}{2}\min\{p,q\}(1 \pm o_n(1))\right)^{3/5}\right) + o_n(1) ~=~ o_n(1),
\end{align*}
where the last equality follows from the assumption that $p,q = \Omega(\log^2 n / n)$, and hence 
\begin{align*}
    nT\exp\left(-2T\left(\frac{n}{2}\min\{p,q\}(1 \pm o_n(1))\right)^{3/5}\right) 
    &= nT \exp\left(-\omega\left(\sqrt{2}T\log n \right)\right) = O\left(\frac{1}{n^c}\right)
\end{align*}
for some absolute constant $c > 0$. Moreover, we have used degree concentration, which introduced the additional additive $o_n(1)$ term in the probability upper bound. Therefore we have
\[
	\Pr\left[|B_{i,0}^t| \le \Ex [|B_{i,0}^t|] +  \sqrt{T}|N_i \cap C_0|^{\frac{4}{5}}, \forall i \in [n] ~\forall t\in[T]\right] \ge 1-o_n(1).
\]
\end{proof}

Define an event $E^*$ as the intersection of the following events:
\begin{enumerate}
\item Concentration of degrees described in \cref{prop:degree-conc};
\item Concentration of number of neighbors in each class described in \cref{prop:class_degree};
\item Lower bounds of $|J_{i,0}|$ and $|J_{i,1}|$ described in \cref{claim:card_J};
\item Upper bounds of $|B_{i,0}^t|$ and $|B_{i,1}^t|$ described in \cref{claim:card_B}
\end{enumerate}
Then a simple union bound shows that with probability at least $1 - o_n(1)$ then event $E^*$ holds. The follows Lemma bounds the growth rate of sum of exponential of Gaussian random variables.

\begin{lemma}[Sum of exponential of Gaussians]
\label{lem:sum_of_exp}
Let $\eta:\mathbb{R} \rightarrow \mathbb{R}$ be a Lipschitz continuous function such that $|\eta(x) - \eta(y)| \le c_1 |x - y|$ for all $x,y$ and $|\eta(0)| \le c_2$ for some absolute constants $c_1, c_2 \ge 0$. Under the event $E^*$ we have
\begin{enumerate}
\item For all $i \in \mathcal{A} \cap C_0$,
\begin{align*}
	&\Omega(np) \le \sum_{j \in N_i \cap C_0}\exp(\eta(s^T f_{(i,j)})) = O(n(p+q)),\\
	&\Omega(nq) \le \sum_{j \in N_i \cap C_1}\exp(\eta(s^T f_{(i,j)})) = O(n(p+q));
\end{align*}
\item For all $i \in \mathcal{A} \cap C_1$,
\begin{align*}
	&\Omega(nq) \le \sum_{j \in N_i \cap C_0}\exp(\eta(s^T f_{(i,j)})) = O(n(p+q)),\\
	&\Omega(np) \le \sum_{j \in N_i \cap C_1}\exp(\eta(s^T f_{(i,j)})) = O(n(p+q)).
\end{align*}
\end{enumerate}
\end{lemma}

\begin{proof}
By the Lipschitz continuity of $\eta$ we know that $|\eta(s^T f_{(i,j)}) - \eta(0)| \le c_1|s^T f_{(i,j)}|$ and hence 
\begin{align}
	&\eta(s^T f_{(i,j)}) \le c_1|s^T f_{(i,j)}| + |\eta(0)| \le c_1|s^T f_{(i,j)}| + c_2, \; \forall (i,j) \in \mathcal{E},\\
	&\eta(s^T f_{(i,j)}) \ge -c_1|s^T f_{(i,j)}| - |\eta(0)| \ge -c_1|s^T f_{(i,j)}| - c_2, \; \forall (i,j) \in \mathcal{E}.
\end{align}
Let $i \in \mathcal{A} \cap C_0$. We have that
\begin{equation}
\label{eq:sum_exp_lb0}
\begin{split}
	\sum_{j \in N_i \cap C_0} \exp(\eta(s^T f_{(i,j)})) 
	&\ge \sum_{j \in J_{i,0}} \exp(\eta(s^T f_{(i,j)})) \ge  \sum_{j \in J_{i,0}} \exp(-c_1|s^T f_{(i,j)}| - c_2) \\
	&\ge \sum_{j \in J_{i,0}} \exp(-c_1\sqrt{10} - c_2) = \Omega(|J_{i,0}|) = \Omega(|N_i \cap C_0|) = \Omega(np),
\end{split}
\end{equation}
and similarly
\begin{equation}
\label{eq:sum_exp_lb1}
	\sum_{j \in N_i \cap C_1} \exp(\eta(s^T f_{(i,j)})) \ge \Omega(|N_i \cap C_1|) = \Omega(nq).
\end{equation}
Therefore it left to obtain the upper bounds. Write
\[
	N_i \cap C_0 = \bar{B}_{i,0} \cup \bigcup_{t\in[T]} B_{i,0}^t
\]
where $\bar{B}_{i,0} := \{j \in N_i \cap C_0 : |s^T f_{(i,j)}| \le 1\}$. It is easy to see that
\begin{equation}
\label{eq:bins0}
	\sum_{j \in \bar{B}_{i,0}} \exp(\eta(s^T f_{(i,j)})) \le \sum_{j \in \bar{B}_{i,0}} \exp(c_1 + c_2) = O(| \bar{B}_{i,0}|) \le O(|N_i \cap C_0|) \le O(n(p+q)).
\end{equation}
On the other hand, we have
\begin{equation}
	 \sum_{t \in [T]}\sum_{j \in B_{i,0}^t}\exp(\eta(s^T f_{(i,j)})) \le  \sum_{t \in [T]}\sum_{j \in B_{i,0}^t}\exp(c_12^t + c_2) =  \sum_{t \in [T]} |B_{i,0}^t|\exp(c_12^t + c_2).
\end{equation}
In order to upper bound the above quantity, note that under event $E^*$ we have $|B_{i,0}^t| \le  E[|B_{i,0}^t|] + \sqrt{T}|N_i \cap C_0|^{4/5}$ for all $t \in [T]$, where
\[
	 E[|B_{i,0}^t|] = \sum_{j \in N_i \cap C_0} P(2^{t-1} \le |s^T f_{(i,j)}| \le 2^t) \le \sum_{j \in N_i \cap C_0}2 P(2^{t-1} \le s^T f_{(i,j)}) \le 2|N_i \cap C_0| \exp(-2^{2t-3}).
\]
Therefore
\begin{equation}
\label{eq:bins1}
\begin{split}
	 &\sum_{t \in [T]} |B_{i,0}^t|\exp(c_12^t + c_2) \\
	 \le~&\sum_{t \in [T]} \left(2|N_i \cap C_0| \exp(-2^{2t-3}) + \sqrt{T}|N_i \cap C_0|^{4/5}\right) \exp(c_12^t + c_2)\\
	 \le~&2|N_i \cap C_0| \sum_{t=1}^{\infty}  \exp(-2^{2t-3} + c_12^t + c_2) + \sum_{t \in [T]} |N_i \cap C_0|^{4/5} \exp(c_12^T + c_2)\\
	 \le~&O(|N_i \cap C_0|) + o(n(p+q))\\
	 \le~&O(n(p+q)).
\end{split}
\end{equation}
The third inequality in the above follows from
\begin{itemize}
\item The infinite series $\sum_{t=1}^{\infty} \exp(-2^{2t-3} + c_12^t + c_2) = c_3$ for some absolute constant $c_3 \ge 0$, because the series converges absolutely for any constants $c_1,c_2 \ge 0$;
\item The finite sum $\sum_{t \in [T]}\sqrt{T} |N_i \cap C_0|^{4/5} \exp(c_12^T + c_2) = o(n(p+q))$ because
\begin{align*}
	\log\left(T^{3/2} \exp(c_12^T + c_2)\right) 
	&= \frac{3}{2} \log \left\lceil \log_2\left(10\sqrt{\log(n(p+q))}\right)\right\rceil + c_12^{ \left\lceil \log_2\left(10\sqrt{\log(n(p+q))}\right)\right\rceil} + c_2\\
	&\le \frac{3}{2} \log \left\lceil \log_2\left(10\sqrt{\log(n(p+q))}\right)\right\rceil  + 20 c_1\sqrt{\log(n(p+q))} + c_2\\
	&\le O\left(\frac{1}{c} \log(n(p+q))\right)
\end{align*}
for any absolute constant $c>0$. Pick $c = 6$ we see that $T^{3/2}\exp(c_12^T + c_2) \le O((n(p+q)^{1/6})$ and hence we get
\begin{align*}
	\sum_{t \in [T]}\sqrt{T} |N_i \cap C_0|^{4/5} \exp(c_12^T + c_2) 
	&= |N_i \cap C_0|^{4/5} T^{3/2}\exp(c_12^T + c_2) \\
	&\le |N_i \cap C_0|^{4/5} \cdot O((n(p+q)^{1/6}) = o(n(p+q)).
\end{align*}
\end{itemize}
Combine \eqref{eq:bins0} and \eqref{eq:bins1} we get that
\begin{equation}
\label{eq:sum_exp_ub0}
	\sum_{j \in N_i \cap C_0}  \exp(\eta(s^T f_{(i,j)})) \le \sum_{j \in \bar{B}_{i,0}}  \exp(\eta(s^T f_{(i,j)})) + \sum_{t \in [T]}\sum_{j \in B_{i,0}^t}  \exp(\eta(s^T f_{(i,j)})) \le O(n(p+q)).
\end{equation}
By repeating the same argument for $N_i \cap C_1$ we have
\begin{equation}
\label{eq:sum_exp_ub1}
	\sum_{j \in N_i \cap C_1}  \exp(\eta(s^T f_{(i,j)})) \le O(n(p+q)).
\end{equation}
Finally, by combining \eqref{eq:sum_exp_lb0}, \eqref{eq:sum_exp_lb1}, \eqref{eq:sum_exp_ub0} and \eqref{eq:sum_exp_ub1} and noticing that our choice of $i \in \mathcal{A} \cap C_0$ was arbitrary, we obtain the claimed results part 1 of the Lemma. The proof for part 2 of the Lemma follows in the same way.
\end{proof}

\begin{proposition}\label{a-prop:gammas_noisy}
Assume that $\|\nu\|_2 \le K \zeta $ for some $K=\mathcal{O}(1)$. Then, with probability at least $1-o_n(1)$, there exists a subset $\mathcal{A} \subseteq [n]$ with cardinality at least $n - o(n)$ such that for all $i \in \mathcal{A}$ the following hold:
\begin{enumerate}
	\item There is a subset $J_{i,0} \subseteq N_i \cap C_0$ with cardinality at least $\frac{9}{10}|N_i \cap C_0|$, such that $\gamma_{ij} = \Theta(1/|N_i|)$ for all $j \in J_{i,0}$;
	\item There is a subset $J_{i,1} \subseteq N_i \cap C_1$ with cardinality at least $\frac{9}{10}|N_i \cap C_1|$, such that $\gamma_{ij} = \Theta(1/|N_i|)$ for all $j \in J_{i,1}$;
	\item If $p \ge q$ we have 
	\begin{align*}
		\Omega(1) \le \frac{\sum_{j \in N_i \cap C_0} \gamma_{ij}}{\sum_{j \in N_i \cap C_1} \gamma_{ij}} \le O(p/q) \ \  \mbox{if} \ \ i \in C_0,\\
		\Omega(1) \le \frac{\sum_{j \in N_i \cap C_1} \gamma_{ij}}{\sum_{j \in N_i \cap C_0} \gamma_{ij}}\le O(p/q) \ \ \mbox{if} \ \ i \in C_1,	
	\end{align*}
	which implies
	\begin{align*}
		-O\left(\frac{p-q}{p+q}\right) \le \sum_{j \in N_i \cap C_1} \gamma_{ij} - \sum_{j \in N_i \cap C_0} \gamma_{ij} \le 0 \ \  \mbox{if} \ \ i \in C_0,\\
		0 \le  \sum_{j \in N_i \cap C_1} \gamma_{ij} - \sum_{j \in N_i \cap C_0} \gamma_{ij} \le O\left(\frac{p-q}{p+q}\right)  \ \  \mbox{if} \ \ i \in C_1.
	\end{align*}
	\item If $p \le q$ we have 
	\begin{align*}
		\Omega(p/q) \le \frac{\sum_{j \in N_i \cap C_0} \gamma_{ij}}{\sum_{j \in N_i \cap C_1} \gamma_{ij}} \le O(1) \ \  \mbox{if} \ \ i \in C_0,\\
		\Omega(p/q) \le \frac{\sum_{j \in N_i \cap C_1} \gamma_{ij}}{\sum_{j \in N_i \cap C_0} \gamma_{ij}}\le O(1) \ \ \mbox{if} \ \ i \in C_1,	
	\end{align*}
	which implies
	\begin{align*}
		0 \le  \sum_{j \in N_i \cap C_1} \gamma_{ij} - \sum_{j \in N_i \cap C_0} \gamma_{ij} \le O\left(\frac{p-q}{p+q}\right)  \ \  \mbox{if} \ \ i \in C_0,\\
		-O\left(\frac{p-q}{p+q}\right) \le \sum_{j \in N_i \cap C_1} \gamma_{ij} - \sum_{j \in N_i \cap C_0} \gamma_{ij} \le 0 \ \  \mbox{if} \ \ i \in C_1.
	\end{align*}
	\item We have
	\[
	    \sum_{j \in N_i} \gamma_{ij}^2 = \Theta\left(\frac{1}{n(p+q)}\right).
	\]
\end{enumerate}
\end{proposition}

\begin{proof}
To obtain part 1 and part 2 of the Proposition, let $j \in J_{i,0}$, then we may write
\[
    \gamma_{ij} = \frac{\exp(\eta(s^T f_{(i,j)}))}{\sum_{l \in N_i}\exp(\eta(s^T f_{(i,l)}))}
\]
where for $l \in N_i$ we set $\eta(s^Tf_{(i,l)}) := \phi((2\epsilon_i-1)(2\epsilon_l-1)s^T\nu + \zeta s^Tf_{i,l})$. Since $j \in J_{i,0}$ we have that
\begin{equation}
\begin{split}
\label{eq:unif_gamma_top}
    |\eta(s^T f_{(i,j)})| 
    &= |\phi((1-\epsilon_l)s^T\nu + \zeta s^Tf_{i,l})|
    \le L|(1-\epsilon_l)s^T\nu + \zeta s^Tf_{i,l})| + R \\
    &\le 2LK\zeta + \sqrt{10}\zeta + R = O(1),
\end{split}
\end{equation}
and thus $\exp(\eta(s^T f_{(i,j)})) = \Theta(1)$. Moreover, $\eta$ is Lipschitz continuous with Lipschitz constant $L\zeta = O(1)$ and satisfies $|\eta(0)| \le 2L|s^T\nu| \le 2L\|\nu\| \le 2LK\zeta = O(1)$, we may use Lemma~\ref{lem:sum_of_exp} and get
\begin{equation}
\label{eq:unif_gamma_bot}
    \sum_{l \in N_i}\exp(\eta(s^T f_{(i,l)})) = \Theta(|N_i|).
\end{equation}
Combine \eqref{eq:unif_gamma_top} and \eqref{eq:unif_gamma_bot} we get $\gamma_{ij} = \Theta(1/|N_i|)$. Since our choice of $j \in J_{i,0}$ was arbitrary and $|J_{i,0}| \ge \frac{9}{10}|N_i \cap C_0|$, this proves part 1. The proof for part 2 follows in the same way.

To obtain part 3 and part 4 of the Proposition, we consider again
\[
     \gamma_{ij} = \frac{\exp(\eta(s^T f_{(i,j)}))}{\sum_{l \in N_i}\exp(\eta(s^T f_{(i,l)}))}
\]
for the same $\eta$ defined above. By writing
\[
    \frac{\sum_{j \in N_i \cap C_0} \gamma_{ij}}{\sum_{j \in N_i \cap C_1} \gamma_{ij}} = \frac{\sum_{j \in N_i \cap C_0} \exp(\eta(s^T f_{(i,j)}))}{\sum_{j \in N_i \cap C_1} \exp(\eta(s^T f_{(i,j)}))}.
\]
and using Lemma~\ref{lem:sum_of_exp} to bound the numerator and denominator separately we obtain the claimed results.

To see part 5 of the Proposition, we may write
\[
    \sum_{j \in N_i} \gamma_{ij}^2 = \frac{\sum_{j \in N_i} \exp(\eta_2(s^T f_{(i,j)}))}{\left(\sum_{j \in N_i} \exp(\eta_1(s^T f_{(i,j)}))\right)^2}
\]
where
\begin{align*}
\eta_1(s^T f_{(i,j)}) &:= \phi((2\epsilon_i-1)(2\epsilon_j-1)s^T\nu + \zeta s^T f_{(i,j)})),\\
\eta_2(s^T f_{(i,j)}) &:= 2 \cdot \eta_1(s^T f_{(i,j)}).
\end{align*}

The function $\eta_1$ is Lipschitz continuous with Lipschitz constant $2L\zeta = O(1)$ and satisfies $|\eta_1(0)| = 2|\phi(0)| \le 4L|s^T\nu| \le 4L\|\nu\| = O(1)$; the function $\eta_2$ is Lipschitz continuous with Lipschitz constant $L\zeta = O(1)$ and satisfies $|\eta_2(0)| = |\phi(0)| \le 2L|s^T\nu| \le 2L\|\nu\| = O(1)$. Therefore we may use Lemma~\ref{lem:sum_of_exp} and get
\[
    \frac{\sum_{j \in N_i} \exp(\eta_1(s^T f_{(i,j)}))}{\left(\sum_{j \in N_i} \exp(\eta_2(s^T f_{(i,j)}))\right)^2} = \frac{\Theta(n(p+q))}{\Theta(n(p+q))^2} = \Theta\left(\frac{1}{n(p+q)}\right).
\]
\end{proof}

\begin{theorem}\label{a-thm:noisy}
Let $(A,X, E)\sim CSBM(n,p,q,\mu,\nu,\sigma, \zeta)$, and assume that $\|\nu\| \le K \zeta $ for some $K=\mathcal{O}(1)$.
\begin{enumerate}
    \item If $\|\mu\| = \omega\left(\sigma\frac{p+q}{|p-q|}\sqrt{\frac{\log n}{n\max(p,q)}}\right)$, with probability $1-o_n(1)$ graph attention classifies all nodes correctly.
    \item If $\|\mu\| \le K' \sigma \frac{p+q}{|p-q|}\sqrt{\frac{\log n}{n\max(p,q)}(1-\max(p,q))}$ for some constant $K'$  and if $\max(p,q) \le 1-36 \log n / n$, then for any fixed $\|w\| = 1$ graph attention fails to perfectly classify the nodes with probability at least $1-2\exp(-c'(1-\max(p,q))\log n)$ for some constant $c'$.
\end{enumerate}
\end{theorem}

\begin{proof}
    We will prove the results for $p>q$ since the analysis for $q>p$ is similar. We start by proving part 2 of the theorem. 
     
     We will prove that the probability of classifying the nodes in set $\mathcal{A}$ (\cref{a-claim:cal_A}) correctly is very small. Thus the probability of classifying all nodes correctly is very small.
     Let us start with the event of correct classification of nodes in $\mathcal{A}$, which is 
    \begin{align*}
        \left(-\sum_{j\in N_i \cap C_0} \gamma_{ij} + \sum_{j\in N_i \cap C_1} \gamma_{ij}\right)w^T\mu + \sigma \sum_{j\in N_i}\gamma_{ij} w^T g_j < 0, \ \forall i \in \mathcal{A} \cap C_0\\
        \left(-\sum_{j\in N_i \cap C_0} \gamma_{ij} + \sum_{j\in N_i \cap C_1} \gamma_{ij}\right)w^T\mu + \sigma \sum_{j\in N_i} \gamma_{ij} w^T g_j > 0, \ \forall i\in \mathcal{A} \cap  C_1,
    \end{align*}
    Using part 3 of~\cref{a-prop:gammas_noisy}, $\|w\|=1$ and $w^T\mu \le \|\mu\| \le K' \sigma \frac{p+q}{|p-q|}\sqrt{\frac{\log n}{n\max(p,q)}(1-\max(p,q))}$ for some constant $K'$, consider the larger event
    \begin{align*}
         \sigma \sum_{j\in N_i}\gamma_{ij} w^T g_j & < K' \sigma \sqrt{\frac{\log n}{n\max(p,q)}(1-\max(p,q))}, \ \forall i \in \mathcal{A} \cap C_0\\
         \sigma \sum_{j\in N_i} \gamma_{ij} w^T g_j & > -K' \sigma \sqrt{\frac{\log n}{n\max(p,q)}(1-\max(p,q))}, \ \forall i\in \mathcal{A} \cap  C_1,
    \end{align*}
    and the equivalent event
    \begin{align*}
         \max_{i\in \mathcal{A} \cap C_0} \sigma \sum_{j\in N_i}\gamma_{ij} w^T g_j & < K' \sigma \sqrt{\frac{\log n}{n\max(p,q)}(1-\max(p,q))}, \\
         \min_{i\in \mathcal{A} \cap C_1} \sigma \sum_{j\in N_i} \gamma_{ij} w^T g_j & > -K' \sigma \sqrt{\frac{\log n}{n\max(p,q)}(1-\max(p,q))}.
    \end{align*}
    Let's bound the probability of correct classification for $C_0$, i.e.,
    \begin{align*}
        &P\left(\max_{i\in \mathcal{A} \cap C_0} \sum_{j\in N_i}\gamma_{ij} w^T g_j < K' \sqrt{\frac{\log n}{n\max(p,q)}(1-\max(p,q))}\right)\\
        =~& P\left(\max_{i\in \mathcal{A} \cap C_0} \sum_{j\in N_i}\gamma_{ij} w^T g_j < K' \sqrt{\frac{\log n}{np}(1-p)}\right),
    \end{align*}
    the result for $C_1$ is similar. 
    We will utilize Sudakov's minoration inequality~\cite[Section 7.4]{Vershynin:2018} to obtain a lower bound on the expected supremum of the corresponding Gaussian process, and then use Borell's inequality~\cite[Section 2.1]{Adler:2007} to upper bound the probability.
    
    Let $Z_i:=\sum_{j\in N_i}\gamma_{ij} w^T g_j$ for $i\in \mathcal{A} \cap C_0$. 
    To apply Sudakov's minoration result, we also define the canonical metric $d_T(i,j) =\sqrt{\E[(Z_i - Z_j)^2]}$ for any $i,j\in \mathcal{A} \cap C_0$. In what follows we will first compute $\E[(Z_i - Z_j)^2]$ and then the metric. Using \cref{prop:degree-conc}, \cref{prop:temp}, \cref{claim:card_J2} and $p > q$ we have that
    \begin{align*}
        \E[(Z_i - Z_j)^2] 
        & = \E\left[\left(\sum_{l\in N_i}\gamma_{il} w^T g_l\right)^2 + \left(\sum_{l\in N_j}\gamma_{jl} w^T g_l\right)^2 - 2 \left(\sum_{l\in N_i}\gamma_{il} w^T g_l\right)\left(\sum_{l\in N_j}\gamma_{jl} w^T g_l\right)\right] \\ 
        & = \sum_{l\in N_i}\gamma_{il}^2 + \sum_{l\in N_j}\gamma_{jl}^2 - 2\sum_{l\in N_i \cap N_j} \gamma_{il}\gamma_{jl}
        \ge \sum_{l \in (N_i \backslash (N_i \cap N_j))\cap C_0} \gamma_{il}^2 + \sum_{l \in (N_j \backslash (N_i \cap N_j))\cap C_0} \gamma_{jl}^2\\
        & \ge \sum_{l \in \widehat{J}_{ij}\cap C_0} \gamma_{il}^2 
        = |\widehat{J}_{ij}\cap C_0| \cdot \Theta\left(\frac{1}{n^2(p+q)^2}\right)
        = \Theta\left(\frac{|((N_i \cup N_j)\backslash(N_i \cap N_j))\cap C_0|}{n^2(p+q)^2}\right)\\
        & = \Theta\left(\frac{\frac{n}{2}(p-p^2)}{n^2(p+q)^2}\right)
        =\Theta\left(\frac{n(p-p^2)}{n^2p^2}\right)
        = \Theta\left(\frac{1}{np}(1-p)\right).
    \end{align*}
    Therefore we have that
    $$
    d_T(i,j) \ge \Theta\left(\sqrt{\frac{1}{np}}\sqrt{1-p}\right).
    $$
    Using this result in Sudakov's minoration inequality, we obtain that
    $$\E[\max_i Z_i] \ge c_1 \sqrt{\frac{\log n}{np}(1-p)}.
    $$
    for some absolute constant $c_1$.
    We now use Borell's inequality~\cite[Section 2.1]{Adler:2007} and the fact that the variance of the Gaussian data after graph attention convolution is $\Theta(1/n(p+q)) = \Theta(1/np)$ (see part 5 of~\cref{a-prop:gammas_noisy}) to obtain that for any $t>0$,
    \begin{align*}
        &P\left(\max_{i\in \mathcal{A} \cap C_0} Z_i \le \E \max_{i\in \mathcal{A} \cap C_0} Z_i - t\right) \le 2\exp(-Ct^2np)\\
        \implies \quad & P\left(\max_{i\in \mathcal{A} \cap C_0} Z_i \le c_1 \sqrt{\frac{\log n}{np}(1-p)} - t\right) \le 2\exp(-Ct^2 np)
    \end{align*}
    for some absolute constant $C>0$. Now, for an appropriate constant $K'$ we may set $t$ such that
    \begin{align*}
    t & =c_1 \sqrt{\frac{\log n}{n(p+q)}(1-p)} - K'  \sqrt{\frac{\log n}{np}(1-p)}\\ 
    & = \Omega\left(\sqrt{\frac{\log n}{np}(1-p)}\right).
    \end{align*}
    Use this $t$ in the above probability we get that for some constant $c_2>0$
    \begin{align*}
        P\left(\max_{i\in \mathcal{A} \cap C_0} Z_i \le K' \sqrt{\frac{\log n}{np}(1-p)}\right)\le \exp(-c_2(1-p)\log n).
    \end{align*}
    This means that the event of classifying all nodes in $\mathcal{A} \cap C_0$ correctly has probability at most $\exp(-c_2(1-p)\log n)$. The same holds for nodes in $\mathcal{A} \cap C_1$. Thus the probability of classifying all nodes correctly in $\mathcal{A}$ is at most $\exp(-c_3(1-p)\log n)$ for some constant $c_3 > 0$. Therefore, the probability of classifying all nodes correctly is at most $\exp(-c_3(1-p)\log n)$.
    
    We will now prove part 1. We will prove the result for $p>q$, since the analysis for $q>p$ is similar.
    Define $w:=\sign(p-q)\mu/\|\mu\|$. Also, set $\Psi=0$, which means that all attention coefficients are exactly uniform and graph attention reduces to graph convolution. It is enough to prove the result for this setting of attention coefficients since we will see that this setting matches the lower bound of misclassification from part 2 of the theorem if $p\in [\log^2 n / n, 1 - \epsilon]$ for any constant $\epsilon\in(0,1)$. Therefore, other attention architectures will only offer negligible improvement over graph convolution, if any.
    
    Let $x_i=(2\epsilon_i - 1) \mu + \sigma g_i$, where $g_i \sim N(0,I)$. Let $i\in C_0$ and $\tilde{g}_i\sim N(0,1)$. Using $\gamma_{ij}=1/|N_i|$ and concentration of $|N_i|$ from~\cref{prop:degree-conc} we get
    \begin{align*}
        \hat x_i & = \sum_{j\in [n]} \tilde{A}_{ij}\gamma_{ij} w^T x_j \\ 
                 & = -\sum_{j\in N_i \cap C_0} \frac{2}{n(p+q)}\|\mu\|(1\pm o_n(1)) + \sum_{j\in N_i \cap C_1} \frac{1}{n(p+q)}\|\mu\|(1\pm o_n(1)) \\
                 & \ \ \  + \sigma \sum_{j\in N_i \cap C_0} \frac{2}{n(p+q)} \tilde{g}_i (1\pm o_n(1)) + \sigma \sum_{j\in N_i \cap C_1} \frac{1}{n(p+q)} \tilde{g}_i(1\pm o_n(1))
    \end{align*}
    Let us first work with the sums for $\|\mu\|$. Using~\cref{prop:class_degree} we have that 
    $$
        \sum_{j\in N_i \cap C_0} \frac{2}{n(p+q)}\|\mu\|(1\pm o_n(1)) = \frac{p}{p+q}\|\mu\|(1\pm o_n(1)) 
    $$
    and
    $$
        \sum_{j\in N_i \cap C_1} \frac{1}{n(p+q)}\|\mu\|(1\pm o_n(1))  = \frac{q}{p+q}\|\mu\|(1\pm o_n(1))
    $$
    Putting the two sums for $\|\mu\|$ together we have that 
    $$
    -\sum_{j\in N_i \cap C_0} \frac{2}{n(p+q)}\|\mu\|(1\pm o_n(1)) + \sum_{j\in N_i \cap C_1} \frac{1}{n(p+q)}\|\mu\| = -\frac{p-q}{p+q}\|\mu\|(1\pm o_n(1))
    $$
    Let us now work with the sum of noise over $N_i\cap C_0$. This is a sum of $|N_i\cap C_0|$ standard normals. From Theorem 2.6.3. (General Hoeffding’s inequality) and from concentration of $|N_i\cap C_0|$ from~\cref{prop:class_degree} we have that 
    $$
    P\left( \left|\sum_{j\in N_i \cap C_0} \Theta\left(\frac{1}{n(p+q)}\right) \tilde{g}_i\right| \ge \sqrt{\frac{10 C^2 \log{n}}{n (p+q) c}} \right) \le 2\exp\left( -10\log{n}\right),
    $$
    where $c$ is a constant, and $C$ is the sub-Gaussian constant for $\tilde{g}_i$. Taking a union bound over all $i\in C_0$, we have that with probability $1-o_n(1)$ we have that 
    $$\left|\sum_{j\in N_i \cap C_0} \Theta\left(\frac{1}{n(p+q)}\right) \tilde{g}_i\right| < \sqrt{\frac{10 C^2 \log{n}}{n (p+q) c}} \ \forall i\in C_0.
    $$
    Using similar concentration arguments we get that the second sum of the noise over $N_i \cap C_1$ is of similar order. Thus, since $\|\mu\|\ge \omega\left(\sigma \frac{p+q}{|p-q|}\sqrt{\frac{\log{n}}{n(p+q)}}\right)$ (because $\max(p,q) = \Theta(p+q)$) we get 
    $$
        \hat x_i =-\|\mu\|(1\pm o_n(1)) + o(\|\mu\|).
    $$
    with probability $1-o_n(1)$. Therefore, with high probability nodes in $C_0$ are correctly classified. Using the same procedure for nodes in $C_1$ we get that these nodes are also classified correctly.

\end{proof}
\section{More experiments on real data}

In~\cref{a-table:1} and~\cref{a-table:2} we present the results for all classes for the data in the main paper. 

\begin{table}[ht!]
\caption{Percentages of intra- and inter-mass (intra-m and inter-m respectively) allocation for graph attention (GA) and graph convolution (GC), and test accuracy. We illustrate results for class $0$ and $1$ of each dataset, the rest are shown in the appendix. The first two rows of each dataset correspond to class $0$ and the latter to correspond to class $1$.} \vspace{0.1cm}
\centering
 \begin{tabular}{||c | c c c c c ||} 
 \hline
 data & class & method & intra-m & inter-m  & acc.\\ [0.5ex] 
 \hline\hline
\multirow{20}{*}{\rotatebox[origin=c]{90}{Amzn Co.}} & \multirow{2}{*}{$0$} & GC & $98.7$ & $1.3$ & $96.8$ \\
& & GA & $98.1$ & $1.9$ & $96.7$ \\ \cline{2-6}
& \multirow{2}{*}{$1$} & GC & $93.6$ & $6.4$ & $91.4$ \\
& & GA & $93.3$ & $6.7$ & $88.7$ \\ \cline{2-6}
& \multirow{2}{*}{$2$} & GC & $98.1$ & $1.9$ & $95.8$ \\
& & GA & $97.8$ & $2.2$ & $92.1$ \\ \cline{2-6}
& \multirow{2}{*}{$3$} & GC & $97.5$ & $2.5$ & $96.0$ \\
& & GA & $96.0$ & $4.0$ & $96.0$ \\ \cline{2-6}
& \multirow{2}{*}{$4$} & GC & $89.4$ & $10.6$ & $89.7$ \\
& & GA & $89.4$ & $10.6$ & $83.3$ \\ \cline{2-6}
& \multirow{2}{*}{$5$} & GC & $99.6$ & $0.4$ & $97.8$ \\
& & GA & $99.5$ & $0.5$ & $98.0$ \\ \cline{2-6}
& \multirow{2}{*}{$6$} & GC & $97.0$ & $3.0$ & $96.4$ \\
& & GA & $96.5$ & $3.5$ & $96.5$ \\ \cline{2-6}
& \multirow{2}{*}{$7$} & GC & $99.1$ & $0.9$ & $96.8$ \\
& & GA & $98.5$ & $1.5$ & $94.9$ \\ \cline{2-6}
& \multirow{2}{*}{$8$} & GC & $91.8$ & $8.2$ & $88.9$ \\
& & GA & $91.7$ & $8.3$ & $86.5$ \\ \cline{2-6}
& \multirow{2}{*}{$9$} & GC & $99.3$ & $0.7$ & $97.9$ \\
& & GA & $99.0$ & $1.0$ & $98.2$ \\
\hline \hline
\multirow{16}{*}{\rotatebox[origin=c]{90}{Amzn Ph.}} & \multirow{2}{*}{$0$} & GC & $99.0$ & $1.0$ & $95.1$ \\
& & GA & $98.8$ & $1.2$ & $96.0$ \\ \cline{2-6}
& \multirow{2}{*}{$1$} & GC & $94.9$ & $5.1$ & $94.2$ \\
& & GA & $95.5$ & $4.5$ & $89.0$ \\ \cline{2-6}
& \multirow{2}{*}{$2$} & GC & $99.4$ & $0.6$ & $96.5$ \\
& & GA & $99.1$ & $0.9$ & $92.8$ \\ \cline{2-6}
& \multirow{2}{*}{$3$} & GC & $93.6$ & $6.4$ & $91.5$ \\
& & GA & $91.5$ & $8.5$ & $88.6$ \\ \cline{2-6}
& \multirow{2}{*}{$4$} & GC & $96.2$ & $3.8$ & $88.4$ \\
& & GA & $94.5$ & $5.5$ & $88.6$ \\ \cline{2-6}
& \multirow{2}{*}{$5$} & GC & $99.7$ & $0.3$ & $98.6$ \\
& & GA & $99.6$ & $0.4$ & $95.1$ \\ \cline{2-6}
& \multirow{2}{*}{$6$} & GC & $94.3$ & $5.7$ & $86.8$ \\
& & GA & $92.7$ & $7.3$ & $82.4$ \\ \cline{2-6}
& \multirow{2}{*}{$7$} & GC & $95.7$ & $4.3$ & $95.5$ \\
& & GA & $94.0$ & $6.0$ & $95.5$ \\
\hline
\end{tabular}
\label{a-table:1}
\end{table}

\begin{table}[ht!]
\caption{Percentages of intra- and inter-mass (intra-m and inter-m respectively) allocation for graph attention (GA) and graph convolution (GC), and test accuracy. We illustrate results for class $0$ and $1$ of each dataset, the rest are shown in the appendix. The first two rows of each dataset correspond to class $0$ and the latter to correspond to class $1$.} \vspace{0.1cm}
\centering
 \begin{tabular}{||c | c c c c c ||} 
 \hline
 data & class & method & intra-m & inter-m  & acc.\\ [0.5ex] 
 \hline\hline
\multirow{14}{*}{\rotatebox[origin=c]{90}{Cora}} & \multirow{2}{*}{$0$} & GC & $94.2$ & $5.8$ & $89.5$ \\
& & GA & $94.7$ & $5.3$ & $88.6$ \\ \cline{2-6}
& \multirow{2}{*}{$1$} & GC & $97.3$ & $2.7$ & $92.3$ \\
& & GA & $97.4$ & $2.6$ & $93.7$ \\ \cline{2-6}
& \multirow{2}{*}{$2$} & GC & $97.9$ & $2.1$ & $86.9$ \\
& & GA & $98.1$ & $1.9$ & $90.8$ \\ \cline{2-6}
& \multirow{2}{*}{$3$} & GC & $93.6$ & $6.4$ & $70.7$ \\
& & GA & $94.1$ & $5.9$ & $73.3$ \\ \cline{2-6}
& \multirow{2}{*}{$4$} & GC & $96.3$ & $3.7$ & $88.8$ \\
& & GA & $96.6$ & $3.4$ & $88.8$ \\ \cline{2-6}
& \multirow{2}{*}{$5$} & GC & $96.8$ & $3.2$ & $91.7$ \\
& & GA & $96.7$ & $3.3$ & $91.3$ \\ \cline{2-6}
& \multirow{2}{*}{$6$} & GC & $98.1$ & $1.9$ & $94.9$ \\
& & GA & $98.2$ & $1.8$ & $95.1$ \\
\hline \hline
\multirow{6}{*}{\rotatebox[origin=c]{90}{PubMed}} & \multirow{2}{*}{$0$} & GC & $92.2$ & $7.8$ & $82.1$ \\
& & GA & $92.1$ & $7.9$ & $82.8$ \\ \cline{2-6}
& \multirow{2}{*}{$1$} & GC & $91.4$ & $8.6$ & $58.8$ \\
& & GA & $90.6$ & $9.4$ & $59.5$ \\ \cline{2-6}
& \multirow{2}{*}{$2$} & GC & $89.4$ & $10.6$ & $60.8$ \\
& & GA & $89.1$ & $10.9$ & $61.8$ \\
\hline \hline
\multirow{12}{*}{\rotatebox[origin=c]{90}{CiteSeer}} & \multirow{2}{*}{$0$} & GC & $94.3$ & $5.7$ & $92.0$ \\
& & GA & $94.4$ & $5.6$ & $91.6$ \\ \cline{2-6}
& \multirow{2}{*}{$1$} & GC & $92.6$ & $7.4$ & $82.6$ \\
& & GA & $92.8$ & $7.2$ & $82.8$ \\ \cline{2-6}
& \multirow{2}{*}{$2$} & GC & $92.3$ & $7.7$ & $86.2$ \\
& & GA & $92.7$ & $7.3$ & $85.2$ \\ \cline{2-6}
& \multirow{2}{*}{$3$} & GC & $93.9$ & $6.1$ & $80.0$ \\
& & GA & $93.8$ & $6.2$ & $78.4$ \\ \cline{2-6}
& \multirow{2}{*}{$4$} & GC & $94.7$ & $5.3$ & $85.9$ \\
& & GA & $94.9$ & $5.1$ & $84.9$ \\ \cline{2-6}
& \multirow{2}{*}{$5$} & GC & $96.2$ & $3.8$ & $86.5$ \\
& & GA & $96.1$ & $3.9$ & $85.7$ \\
\hline
\end{tabular}
\label{a-table:2}
\end{table}

\section{Experiments on real data without edge features}

In the real experiments in the main paper we split the features of the dataset into half. The first half is used for node features and the second half is used for edge features. This leaves a small gap compared to the main setting in the original paper~\cite{Velickovic2018GraphAN} and also in our Preliminaries section in the main paper. Here we re-do the same experiments but without splitting the features and using the original setting in~\cite{Velickovic2018GraphAN}. The results are given in~\cref{a-table:3} and~\cref{a-table:4}. The results and the conclusion are similar in this setting as well.

\begin{table}[ht!]
\caption{Percentages of intra- and inter-mass (intra-m and inter-m respectively) allocation for graph attention (GA) and graph convolution (GC), and test accuracy. We illustrate results for class $0$ and $1$ of each dataset, the rest are shown in the appendix. The first two rows of each dataset correspond to class $0$ and the latter to correspond to class $1$.} \vspace{0.1cm}
\centering
 \begin{tabular}{||c | c c c c c ||} 
 \hline
 data & class & method & intra-m & inter-m  & acc.\\ [0.5ex] 
 \hline\hline
\multirow{20}{*}{\rotatebox[origin=c]{90}{computers}} & \multirow{2}{*}{$0$} & GC & $98.7$ & $1.3$ & $96.8$ \\
& & GA & $98.5$ & $1.5$ & $96.9$ \\ \cline{2-6}
& \multirow{2}{*}{$1$} & GC & $93.6$ & $6.4$ & $91.7$ \\
& & GA & $94.5$ & $5.5$ & $91.7$ \\ \cline{2-6}
& \multirow{2}{*}{$2$} & GC & $98.1$ & $1.9$ & $97.5$ \\
& & GA & $98.1$ & $1.9$ & $97.1$ \\ \cline{2-6}
& \multirow{2}{*}{$3$} & GC & $97.5$ & $2.5$ & $96.0$ \\
& & GA & $97.1$ & $2.9$ & $96.0$ \\ \cline{2-6}
& \multirow{2}{*}{$4$} & GC & $89.4$ & $10.6$ & $90.3$ \\
& & GA & $90.7$ & $9.3$ & $86.5$ \\ \cline{2-6}
& \multirow{2}{*}{$5$} & GC & $99.6$ & $0.4$ & $98.7$ \\
& & GA & $99.6$ & $0.4$ & $98.6$ \\ \cline{2-6}
& \multirow{2}{*}{$6$} & GC & $97.0$ & $3.0$ & $96.4$ \\
& & GA & $97.2$ & $2.8$ & $96.4$ \\ \cline{2-6}
& \multirow{2}{*}{$7$} & GC & $99.1$ & $0.9$ & $97.7$ \\
& & GA & $98.8$ & $1.2$ & $97.1$ \\ \cline{2-6}
& \multirow{2}{*}{$8$} & GC & $91.8$ & $8.2$ & $91.5$ \\
& & GA & $93.2$ & $6.8$ & $90.7$ \\ \cline{2-6}
& \multirow{2}{*}{$9$} & GC & $99.3$ & $0.7$ & $98.2$ \\
& & GA & $99.2$ & $0.8$ & $98.3$ \\
\hline \hline
\multirow{16}{*}{\rotatebox[origin=c]{90}{photo}} & \multirow{2}{*}{$0$} & GC & $99.0$ & $1.0$ & $95.3$ \\
& & GA & $98.9$ & $1.1$ & $95.6$ \\ \cline{2-6}
& \multirow{2}{*}{$1$} & GC & $94.9$ & $5.1$ & $95.6$ \\
& & GA & $95.2$ & $4.8$ & $91.2$ \\ \cline{2-6}
& \multirow{2}{*}{$2$} & GC & $99.4$ & $0.6$ & $96.9$ \\
& & GA & $99.3$ & $0.7$ & $96.9$ \\ \cline{2-6}
& \multirow{2}{*}{$3$} & GC & $93.6$ & $6.4$ & $92.3$ \\
& & GA & $94.1$ & $5.9$ & $91.4$ \\ \cline{2-6}
& \multirow{2}{*}{$4$} & GC & $96.2$ & $3.8$ & $96.1$ \\
& & GA & $96.5$ & $3.5$ & $94.2$ \\ \cline{2-6}
& \multirow{2}{*}{$5$} & GC & $99.7$ & $0.3$ & $97.5$ \\
& & GA & $99.7$ & $0.3$ & $98.3$ \\ \cline{2-6}
& \multirow{2}{*}{$6$} & GC & $94.3$ & $5.7$ & $95.5$ \\
& & GA & $95.2$ & $4.8$ & $91.1$ \\ \cline{2-6}
& \multirow{2}{*}{$7$} & GC & $95.7$ & $4.3$ & $95.7$ \\
& & GA & $96.1$ & $3.9$ & $95.6$ \\
\hline
\end{tabular}
\label{a-table:3}
\end{table}

\begin{table}[ht!]
\caption{Percentages of intra- and inter-mass (intra-m and inter-m respectively) allocation for graph attention (GA) and graph convolution (GC), and test accuracy. We illustrate results for class $0$ and $1$ of each dataset, the rest are shown in the appendix. The first two rows of each dataset correspond to class $0$ and the latter to correspond to class $1$.} \vspace{0.1cm}
\centering
 \begin{tabular}{||c | c c c c c ||} 
 \hline
 data & class & method & intra-m & inter-m  & acc.\\ [0.5ex] 
 \hline\hline
\multirow{14}{*}{\rotatebox[origin=c]{90}{Cora}} & \multirow{2}{*}{$0$} & GC & $94.2$ & $5.8$ & $90.3$ \\
& & GA & $94.7$ & $5.3$ & $89.8$ \\ \cline{2-6}
& \multirow{2}{*}{$1$} & GC & $97.3$ & $2.7$ & $94.2$ \\
& & GA & $97.7$ & $2.3$ & $94.2$ \\ \cline{2-6}
& \multirow{2}{*}{$2$} & GC & $97.9$ & $2.1$ & $90.3$ \\
& & GA & $98.1$ & $1.9$ & $91.1$ \\ \cline{2-6}
& \multirow{2}{*}{$3$} & GC & $93.6$ & $6.4$ & $74.9$ \\
& & GA & $94.1$ & $5.9$ & $79.7$ \\ \cline{2-6}
& \multirow{2}{*}{$4$} & GC & $96.3$ & $3.7$ & $89.4$ \\
& & GA & $96.6$ & $3.4$ & $90.5$ \\ \cline{2-6}
& \multirow{2}{*}{$5$} & GC & $96.8$ & $3.2$ & $93.0$ \\
& & GA & $96.7$ & $3.3$ & $92.7$ \\ \cline{2-6}
& \multirow{2}{*}{$6$} & GC & $98.1$ & $1.9$ & $95.2$ \\
& & GA & $98.4$ & $1.6$ & $95.4$ \\
\hline \hline
\multirow{6}{*}{\rotatebox[origin=c]{90}{PubMed}} & \multirow{2}{*}{$0$} & GC & $92.2$ & $7.8$ & $84.3$ \\
& & GA & $92.0$ & $8.0$ & $85.7$ \\ \cline{2-6}
& \multirow{2}{*}{$1$} & GC & $91.4$ & $8.6$ & $59.1$ \\
& & GA & $90.8$ & $9.2$ & $64.1$ \\ \cline{2-6}
& \multirow{2}{*}{$2$} & GC & $89.4$ & $10.6$ & $61.5$ \\
& & GA & $89.0$ & $11.0$ & $61.6$ \\
\hline \hline
\multirow{12}{*}{\rotatebox[origin=c]{90}{CiteSeer}} & \multirow{2}{*}{$0$} & GC & $94.3$ & $5.7$ & $92.7$ \\
& & GA & $94.3$ & $5.7$ & $92.3$ \\ \cline{2-6}
& \multirow{2}{*}{$1$} & GC & $92.6$ & $7.4$ & $83.2$ \\
& & GA & $93.1$ & $6.9$ & $83.8$ \\ \cline{2-6}
& \multirow{2}{*}{$2$} & GC & $92.3$ & $7.7$ & $85.7$ \\
& & GA & $92.7$ & $7.3$ & $85.7$ \\ \cline{2-6}
& \multirow{2}{*}{$3$} & GC & $93.9$ & $6.1$ & $79.6$ \\
& & GA & $93.9$ & $6.1$ & $79.8$ \\ \cline{2-6}
& \multirow{2}{*}{$4$} & GC & $94.7$ & $5.3$ & $87.5$ \\
& & GA & $94.6$ & $5.4$ & $87.3$ \\ \cline{2-6}
& \multirow{2}{*}{$5$} & GC & $96.2$ & $3.8$ & $87.6$ \\
& & GA & $96.3$ & $3.7$ & $88.2$ \\
\hline
\end{tabular}
\label{a-table:4}
\end{table}





\end{document}